\documentclass[letterpaper, 10 pt, journal, twoside]{ieeetran}  
\usepackage{amsmath}

\usepackage{amsthm}
\usepackage[font=footnotesize,labelfont=normalfont]{caption}
\usepackage{subcaption}
\usepackage{xcolor}
\usepackage{hyperref}
\hypersetup{
    colorlinks=true,
    linkcolor=blue,
    filecolor=magenta,      
    urlcolor=cyan,
    pdftitle={Overleaf Example},
    pdfpagemode=FullScreen,
    }
\usepackage{algpseudocode}
\usepackage[linesnumbered,ruled,vlined]{algorithm2e}
\usepackage{amsfonts}
\usepackage{svg}
\usepackage{cite}

\DeclareMathOperator*{\argmin}{arg\,min}
\newcommand{\irj}{{}^i\mathbf{r}_j}
\newcommand{\newtext}[1]{{\color{black}{{#1}}}}
\usepackage{xcolor}
\renewcommand\fbox{\fcolorbox{blue}{white}}
\newtheorem{theorem}{Theorem}
\theoremstyle{definition}
\newtheorem{definition}{Definition}[section]
\usepackage{changepage}
\usepackage{siunitx}
\usepackage{tikz}
\def\IEEEtitletopspace{3pt}
\usepackage[top=57pt,bottom=43pt,left=48pt,right=48pt]{geometry}

\title{Robust Trajectory Generation and Control for Quadrotor Motion Planning with Field-of-View Control Barrier Certification}

\author{Lishuo Pan$^{1}$, Mattia Catellani$^{2}$, Lorenzo Sabattini$^{2}$, Nora Ayanian$^{1}$%
\thanks{Manuscript received June 11, 2025; Revised September 17, 2025; Accepted October 20, 2025. 
This paper was recommended for publication by Editor Tamim Asfour and Editor M. Ani Hsieh upon evaluation of the Associate Editor and Reviewers' comments.
This work was supported by NSF grants 2317145, 2311967, 2330942.} 
\thanks{$^{1}$Lishuo Pan and Nora Ayanian are with the Department of Computer Science, Brown University, Providence, RI 02912 USA. 
        Email: {\tt\footnotesize \{lishuo\_pan, nora\_ayanian\}@brown.edu}.}\\ %
\thanks{$^{2}$Mattia Catellani and Lorenzo Sabattini are with the Department of Sciences and Methods for Engineering, University of Modena and Reggio Emilia, 41121 Modena, Italy.
        Email: {\tt\footnotesize \{mattia.catellani, lorenzo.sabattini\}@unimore.it}.}%
}

\begin{document}

\maketitle

\begin{abstract}
Many approaches to multi-robot coordination are susceptible to failure due to communication loss and uncertainty in estimation. We present a real-time communication-free distributed navigation algorithm certified by control barrier functions, that models and controls the onboard sensing behavior to keep neighbors in the limited field of view for position estimation. The approach is robust to temporary tracking loss and directly synthesizes control to stabilize visual contact through control Lyapunov-barrier functions. The main contributions of this paper are a continuous-time robust trajectory generation and control method certified by control barrier functions for distributed multi-robot systems and a discrete optimization procedure, namely, MPC-CBF, to approximate the certified controller. In addition, we propose a linear surrogate of high-order control barrier function constraints and use sequential quadratic programming to solve MPC-CBF efficiently. 

\end{abstract}

\IEEEpeerreviewmaketitle

\section{Introduction}
Multi-robot systems, such as those used in search and rescue~\cite{drew2021multi}, active target tracking~\cite{liu2024multi}, and collaborative transportation~\cite{li2023nonlinear}, demand real-time distributed coordination solutions robust to communication compromises. 
Communication is vulnerable to adversarial attacks and faces challenges such as dropped messages, delays, and scalability~\cite{gielis2022critical}. In contrast, onboard sensing to estimate neighbors' states is robust to compromised communication.  However, one major challenge is dealing with  imperfect perception. 
For instance, onboard cameras typically have a restricted angular field of view, leading to a trade-off between task completion and neighbor detection. 
Although $360^{\circ}$ cameras offer an omnidirectional field of view, object tracking on $360^{\circ}$ images remains challenging due to significant distortions and stitching artifacts~\cite{huang2023360vot}. 
In addition, the limited frame rate and uncertainty pose challenges for practical applications. For instance, a robot may lose track of a neighbor due to image blur caused by the vehicle motion or inaccurate estimation due to measurement noise. 

In this work, we present a real-time robust trajectory generation and control strategy for navigation tasks in distributed multi-robot systems that respects visual contact in communication-denied environments. 
Keeping neighbors in the field of view during navigation is challenging due to the abovementioned limitations. When estimation uncertainty is present or it is infeasible to track all neighbors, temporary compromises of field-of-view constraints are inevitable. A robust controller is demanded to regain the visual contact between robots in such scenarios.

Our robust control strategy utilizes control barrier functions (CBFs)~\cite{xiao2021hoclbf} to maintain visual contact with neighbors and regain it after temporary loss. 
\begin{figure}[t]
    \centering
    {\includegraphics[width=0.48\textwidth]{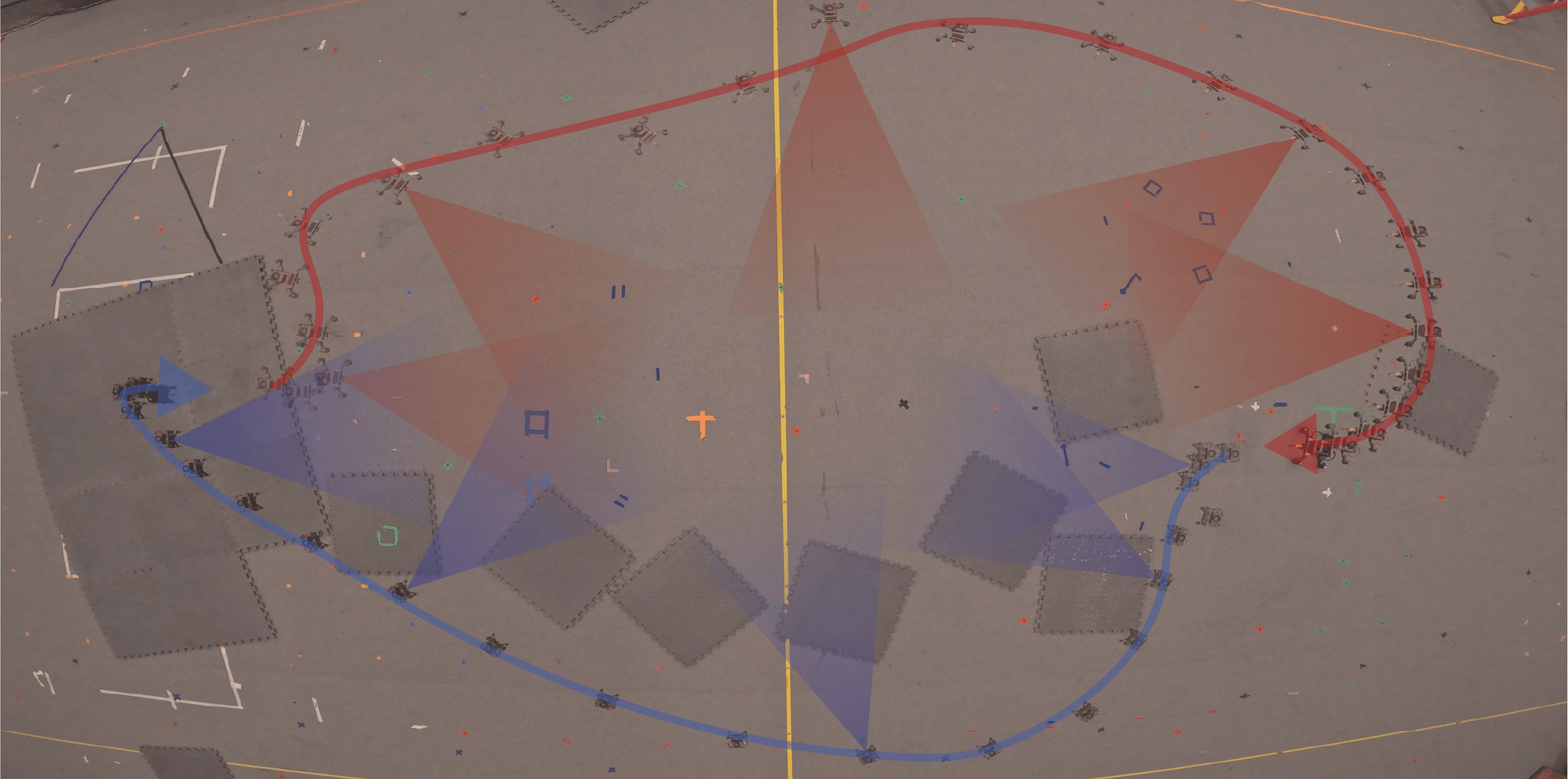}}\\
    \caption{Long exposure top view of $2$ quadrotors navigating with distributed controller respecting field-of-view constraints. The red and blue triangles are the fields of view of UAV1 and UAV2. The sensing ranges extend beyond triangles and are omitted in the figure. Curves represent the robot routes.} 
    \label{fig:demo}
    \vspace{-1em}
\end{figure}
We consider a continuous-time trajectory and control generation problem certified by CBFs. We consider a double integrator model, which requires high-order control barrier functions (HOCBFs) to satisfy the field-of-view constraints. Tracking robustness follows from the Lyapunov-like property of HOCBFs. 
We propose a discrete optimization framework, model predictive control with control barrier functions (MPC-CBF). Since imposing HOCBF constraints for all times in the horizon is intractable, MPC-CBF applies constraints at sampled time stamps to approximate the certified solution.
We introduce a linear surrogate for HOCBF constraints and solve MPC-CBF via sequential quadratic programming (SQP) using a quadratic programming (QP) solver. Our framework jointly generates 
a certified continuous-time trajectory and a controller in real time using piecewise splines. Our algorithm provides inputs up to an arbitrary order of derivatives. 
Our contributions can be summarized as follows:
\begin{itemize}
    \item a real-time distributed controller that maintains visual contact between robots in communication-denied areas and tolerates temporary tracking loss during navigation;
    \item a continuous-time spline-based trajectory and control generation method certified by control barrier functions;
    \item an optimization framework, namely MPC-CBF, that imposes the HOCBF constraints at sampled time stamps and approximates the certified solution. We will open-source our proposed algorithm at \url{https://github.com/LishuoPan/fovmpc}.
\end{itemize}
\newtext{We demonstrate our algorithm in simulations with up to $10$ robots and in physical experiments with $2$ custom-built UAVs, shown in Fig.~\ref{fig:demo}.}


\section{Related Work}
Control barrier functions provide sufficient and necessary conditions to guarantee safety~\cite{ames2019control}. 
Despite success in collision avoidance~\cite{abdi2023safe}, lane keeping and adaptive cruise control~\cite{xu2017correctness}, CBFs synthesize reactive control that can be short-sighted behaviors in planning tasks. For instance, CBF-based controllers can cause deadlocks in multi-robot navigation~\cite{wang2017safety}, and \newtext{aggressive trajectory and heading controls that lead to more failures in visual maintenance} and goal reaching in visual contact navigation~\cite{catellani2023distributed}. 
\newtext{CBFs have recently been applied to limited field-of-view problems: in~\cite{trimarchi2024control} for only static gates in drone racing, in~\cite{zhou2024control} for pursuit settings to track a moving target, and in~\cite{santilli2022multirobot} for multi-robot navigation with triangular sensing. However, the approach in~\cite{santilli2022multirobot} requires centralized intervention and omnidirectional sensing.}
In this work, we combine planning with CBFs to overcome such disadvantages. 


Attempts have been made to combine MPC with CBFs in a discrete-time formulation~\cite{zeng2021safety}. 
Our continuous-time controller improves system responsiveness and stability. 
In continuous-time formulations, a multi-layer controller~\cite{sforni2024receding} solves optimal control with CBF constraints. Our approach provides additional smoothness and derivatives up to an arbitrary order. 
A spline-based trajectory generation method imposes CBF constraints utilizing polygonal cells~\cite{dickson2024spline}; however, it only solves the trajectory and requires separate control synthesis. Instead, our framework solves trajectory and control \emph{concurrently}. 
Our approach has the following advantages. 1) It unifies trajectory and control generation: control inputs can be computed directly from the trajectory. 2) It provides smooth control inputs up to an arbitrary derivative order. 
3) It delivers \emph{continuous-time} trajectory and control, 
thus responding better to delays and stabilizing agile systems. 4) It performs in real time. 


\section{Preliminaries}

\subsection{Bézier Curve}
We use piecewise splines $f(t)$ to impose smoothness requirements in the trajectory generation problem and obtain a trajectory with derivatives up to an arbitrary defined order.
The $i$-th Bézier curve in the piecewise splines $f_{i}: \left[ 0, \tau_{i} \right] \rightarrow \mathbb{R}^{d}$ is parameterized by time, with duration $\tau_{i}$.
The Bézier curve of arbitrary degree $h$ with duration $\tau_{i}$ is defined by $h+1$ control points $\boldsymbol{\mathcal{U}}_{i} = \left[\boldsymbol{u}_{i,0};\ldots;\boldsymbol{u}_{i,h}\right]$. We first construct Bernstein polynomials $\boldsymbol{B}_{v}^{h}\in \mathbb{R}$ of degree $h$\newtext{~\cite{joy2000bernstein}}:
\begin{align}
    \boldsymbol{B}_{v}^{h} = \binom{h}{v} \left(\frac{t}{\tau}\right)^{v} \left(1-\frac{t}{\tau}\right)^{h-v}, \forall t\in \left[0,\tau\right],
\end{align}
where $v=0,1,\cdots,h$.
A $d$-dimensional Bézier curve is defined as $f_{i}(t) = \sum_{v=0}^{h}\boldsymbol{u}_{i,v}\boldsymbol{B}_{i,v}^{h}$ with $\boldsymbol{u}_{i,v}\in\mathbb{R}^{d}$. The finite set of control points $\boldsymbol{\mathcal{U}} = \left[\boldsymbol{\mathcal{U}}_{0}; \ldots; \boldsymbol{\mathcal{U}}_{P-1}\right]$ uniquely characterize a piecewise spline of $P$ Bézier curves and act as decision variables in the trajectory generation problem. The duration of the entire piecewise spline is $\tau = \sum_{i=0}^{P-1} \tau_{i}$.

\subsection{High-Order Control Barrier Functions}
Consider a system in the form
\begin{equation}\label{eq:system}
    \dot{\mathbf{x}} = f(\mathbf{x}) + g(\mathbf{x})\mathbf{u},
\end{equation}
where $f : \mathbb{R}^p \rightarrow \mathbb{R}^p$ and $g : \mathbb{R}^p \rightarrow \mathbb{R}^{p\times m}$ are Lispschitz continuous functions, and $\mathbf{u} \in U \subset \mathbb{R}^m$ is the control input, where $U$ is the set of admissible control inputs for $\mathbf{u}$. Let $\mathcal{C} := \{\mathbf{x} \in \mathbb{R}^p \mid b(\mathbf{x}) \geq 0\}$ be the set of configurations satisfying the safety requirements, i.e., the safe set. 
\begin{definition}[CBF~\cite{ames2014control, ames2019control}]\label{def:cbf}
Given a set $\mathcal{C}$, the function $b : \mathbb{R}^p \rightarrow \mathbb{R}$ is a candidate CBF for system~\eqref{eq:system} if there exists a class $\mathcal{K}$ function
$\alpha$ such that
\begin{equation}\label{eq:cbf_property}
    \sup_{u\in U} [ L_fb(\mathbf{x}) + L_gb(\mathbf{x})\mathbf{u} + \alpha(b(\mathbf{x})) ] \geq 0,
\end{equation}
for all $\mathbf{x}\in \mathcal{C}$, where $L_f$ and $L_g$ are the Lie derivatives~\cite{yano2020theory} along $f$ and $g$, respectively.
According to~\cite{ames2019control}, given a CBF $b$ and a safe set $\mathcal{C}$, any Lipschitz continuous controller $\mathbf{u}(t)$ that satisfies~\eqref{eq:cbf_property} makes the set $\mathcal{C}$ \emph{forward invariant} for system~\eqref{eq:system}, i.e., if $\mathbf{x}(t_0) \in \mathcal{C}$, then $\mathbf{x}(t) \in \mathcal{C}$, $\forall t \geq t_0$.
\end{definition}

If $b$ has relative degree $q > 1$ with respect to system~\ref{eq:system}, the control input $\mathbf{u}(t)$ does not appear in~\eqref{eq:cbf_property};
HOCBFs have been developed for this scenario~\cite{xiao2021high}. Consider functions $\psi_i : \mathbb{R}^p \rightarrow \mathbb{R}$, $i \in \{ 1, \dots, q  \}$ defined as $\psi_i(\mathbf{x}) = \dot\psi_{i-1} (\mathbf{x}) + \alpha_i(\psi_{i-1}(\mathbf{x}))$
where $\psi_0(\mathbf{x}) = b(\mathbf{x})$. 
Furthermore, we define a sequence of sets $\{\mathcal{C}_i\}_{i=1}^{q}$ as $\mathcal{C}_i := \{ \mathbf{x} \in \mathbb{R}^p \mid \psi_{i-1}(\mathbf{x}) \geq 0  \}$.
\begin{definition}[HOCBF~\cite{xiao2021high}]\label{def:hocbf}
Let us consider a sequence of sets $\{\mathcal{C}_{i}\}_{i=1}^{q}$, and a sequence of equations $\{\psi_{i}\}_{i=1}^{q}$ as previously defined.A function $b : \mathbb{R}^n \rightarrow \mathbb{R}$ is a candidate HOCBF of relative degree $q$ for system~\eqref{eq:system} if there exist $(q-i)$-th order differentiable class $\mathcal{K}$ functions $\{\alpha_i\}_{i=1}^{q}$ such that:
\begin{multline}\label{eq:hocbf_def}
    \sup_{\mathbf{u} \in U} [L_f^q b(\mathbf{x}) + [L_gL_f^{q-1}b(\mathbf{x})]\mathbf{u} + O(b(\mathbf{x})) \\
    + \alpha_q(\psi_{q-1}(\mathbf{x}))] \geq 0, 
\end{multline}
for all $\mathbf{x}\in \cap_{i=1}^{q} \mathcal{C}_i$, where $O(\cdot) =\sum_{i=1}^{q-1} L_f^i ( \alpha_{q-i} \circ \psi_{q-i-1})(\mathbf{x})$.
Any Lipschitz continuous controller $\mathbf{u}(t) \in U$ that satisfies~\eqref{eq:hocbf_def} renders the set $\cap_{i=1}^{q} \mathcal{C}_i$ forward invariant (see~\cite[Theorem 4]{xiao2021high}).
\end{definition}

\newtext{The robustness of a HOCBF is accomplished by showing the existence of a high-order control lyapunov-barrier function (HOCLBF)~\cite{xiao2021hoclbf}, which implies asymptotic stability of the set $\cap_{i=1}^{q} \mathcal{C}_i$ by using a set of \textit{extended} class $\mathcal{K}$ functions $\{\alpha_{i}\}_{i=1}^{q}$.} 

\subsection{Neighbors State Estimation}\label{sec:filter}
We adopt particle filters for neighbor position estimation~\cite{catellani2023distributed}. 
The neighbor's position belief is represented by $N_{p} \in \mathbb{N}$ weighted particles, each indicating a possible position.
Particles evolve following a transition function, update weights with new measurements, and are resampled accordingly. 
When the target is outside the field of view, we reduce the weight of particles inside the sensing region $\mathcal{F}_i$ by setting $w_j^k \!\leftarrow\! \varepsilon w_j^k$, where $\varepsilon \!\in\! [0, 1)$. \newtext{Unlike~\cite{catellani2023distributed}, our method accommodates occasional missed detections by uniformly scaling weights of particles within $\mathcal{F}_i$, preserving the normalized distribution. 
}
\section{Problem Formulation}
Consider $N$ homogeneous robots in a communication-denied workspace $\mathcal W$. Let $\mathcal{R}(\mathbf{r}_{i})$ denotes the convex set of points representing robot $i$ at position $\mathbf{r}_{i}\in \mathbb{R}^{3}$. 
Robots generate trajectories and controls concurrently in a decentralized manner to reach goals while maintaining visual contact and avoiding collisions without communication. Each robot estimates others' positions via the field of view of an onboard camera. \newtext{We assume that the initial teammate positions are shared via a pre-mission communication.}
Our optimization problem solves trajectory and control by minimizing the control effort, and the distance to the goal, subject to dynamics, initial state, control continuity, safety corridor, and CBF constraints.

\subsection{Robot model}
\label{sec:robot_model}
We consider a double integrator, for faster computation and a unified planning and control framework using Bézier curve. 
The state includes position, yaw, and corresponding first-order derivatives $\mathbf{x} = [\mathbf{r}; \phi; \dot{\mathbf{r}}; \dot{\phi}]\in \mathbb{R}^{8}$, where $\mathbf{r}\in \mathbb{R}^{3}$, $\phi \in \mathbb{R}$. 
\begin{align}\label{eq:dynamics}
    \dot{\mathbf{x}} &= A\mathbf{x} + B\mathbf{u},
\end{align}
where the control input $\mathbf{u} = [\mathbf{u}_{r}; u_{\phi}] \in \mathbb{R}^{4}$ is the acceleration. The system output is $\mathbf{y} = [\mathbf{r}; \phi]\in \mathbb{R}^{4}$. We denote velocity by $\mathbf{v} = [\dot{\mathbf{r}}; \dot{\phi}] \in \mathbb{R}^{4}$. Physical limits are given by minimum velocity and acceleration $\mathbf{v}_{\mathrm{min}},~\mathbf{a}_{\mathrm{min}}\in\mathbb{R}^{4}$, maximum velocity and acceleration $\mathbf{v}_{\mathrm{max}},~\mathbf{a}_{\mathrm{max}}\in\mathbb{R}^{4}$, 
respectively. $A = [\mathbf{0}, \mathbf{I}; \mathbf{0}, \mathbf{0}] \in \mathbb{R}^{8\times 8}$, $B = [\mathbf{0}; \mathbf{I}] \in \mathbb{R}^{8\times 4}$, 
with $\mathbf{0}, \mathbf{I} \in \mathbb{R}^{4\times 4}$ denoting the zero and identity matrices. At time $t_{0}$, we generate a trajectory $\mathbf{x}(t|t_{0})$ and control $\mathbf{u}(t|t_{0})$ over horizon $\tau$.

\subsection{Sensing model}
Each robot senses via an onboard camera aligned with the body frame, positive $x$-axis. 
We model the robot $i$'s sensing region $\mathcal{F}_i$ as a truncated spherical sector as shown in Fig.~\ref{fig:sensing_model}, limited by a maximum perception range $R_s > 0$, and a minimum safety distance from the other robots $D_s > 0$. Horizontal and vertical fields of view are given by angles $\beta_H, \beta_V \in [0, 2\pi)$, respectively. 
A neighbor's position is observable when inside the field of view. 
Robot $i$ measures the relative position of neighbor $j$ in its body frame, 
i.e., $\irj \!=\! \mathbf{r}_j \!-\! \mathbf{r}_i$. 
We model measurement uncertainty as a zero-mean multivariate Gaussian noise with covariance matrix $R_{\mathrm{m}} \!\in\! \mathbb{R}^{3\times 3}$.
\begin{figure}[t]
    \centering
    \subfloat[Robot sensing region]
    {\includegraphics[width=0.19\textwidth]{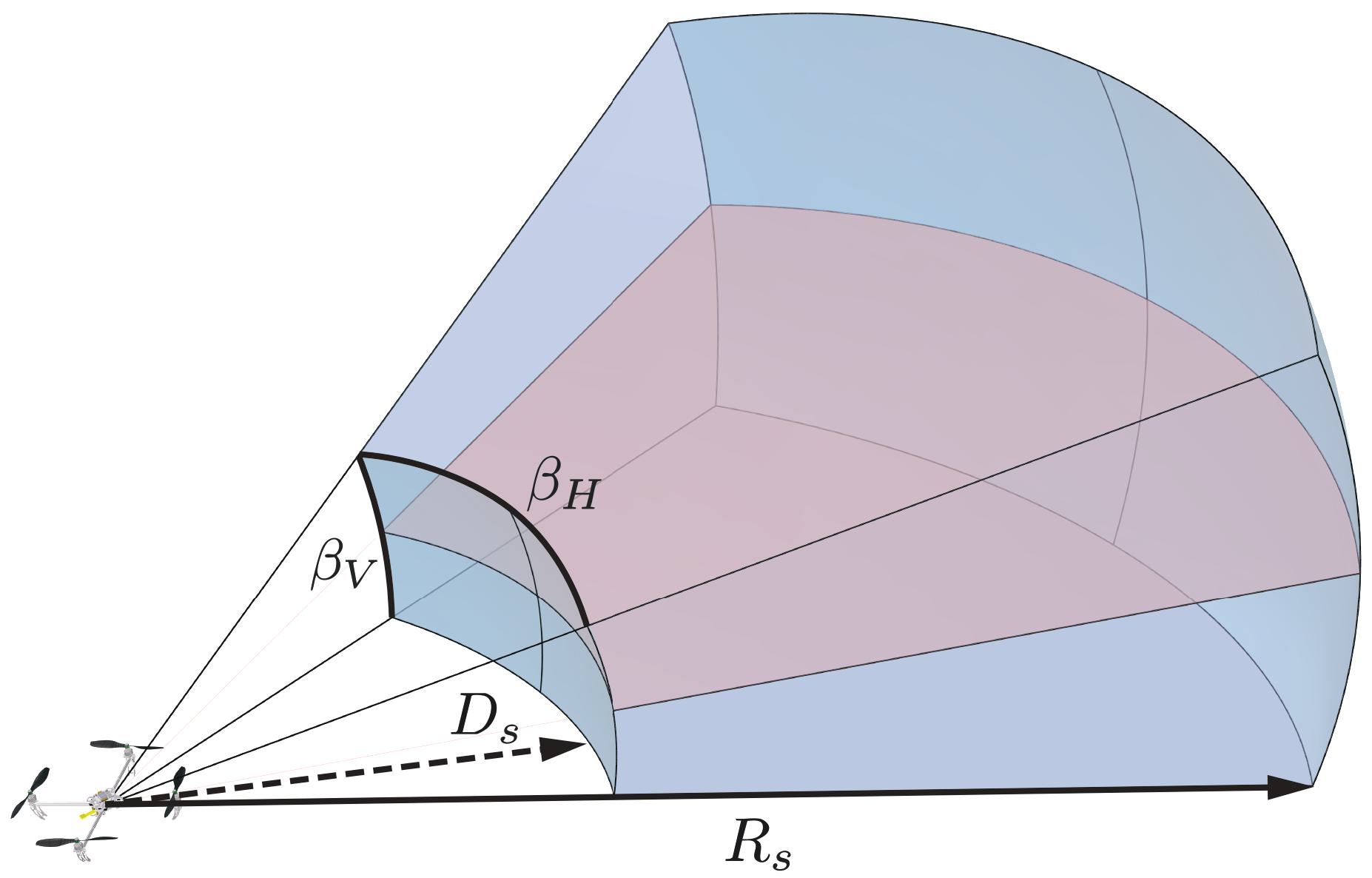}}
    \hspace{2em}
    \subfloat[Top view]
    {\includegraphics[width=0.10\textwidth]{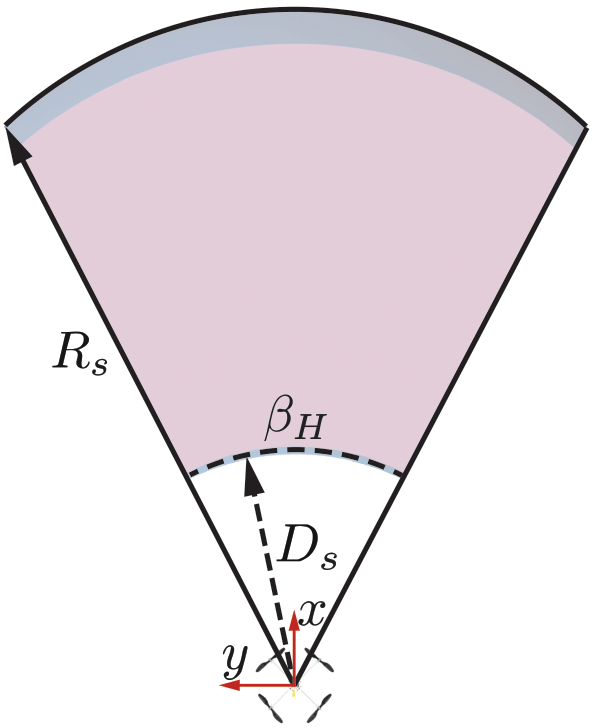}}
    \caption{The sensing region $\mathcal{F}$ of a robot is modeled as truncated spherical sector. $\beta_{H}$, $\beta_{V}$ are the horizontal and vertical field of view angles. $R_{s}$ is the sensing range and $D_{s}$ is the safety distance. 
    The blue volume (or red plane in 2D) is the region where the neighbor can be safely detected. 
    }
    \label{fig:sensing_model}
    \vspace{-2em}
\end{figure}
\section{HOCBFs Design}
We require a robot to maintain a safety distance, visual contact, and perception range with its neighbors. We formulate the following CBF for robot $i$ in the form 
$b(\irj) \geq 0$, $\forall j\in \mathcal{N}_{i}$. Here, we denote the neighbors of robot $i$ as $\mathcal{N}_{i}$ (we consider all the other robots; however, a sensing range can be enforced if desired). 
In this work, we only focus on 2D motion, which can be applied to ground robots or aerial vehicles that fly at the same altitude. 
Thus, the sensing region is a planar angular sector defined by $\beta_H$ (see Fig.~\ref{fig:sensing_model}b). 
The safety distance and range CBFs are defined as follows:
\begin{align}
    &b_{\mathrm{sr}}({}^i\mathbf{r}_j) = \begin{bmatrix}
        {}^ix_j & {}^iy_j \\
        -{}^ix_j & -{}^iy_j
    \end{bmatrix} \begin{bmatrix}
        {}^ix_j \\ {}^iy_j
    \end{bmatrix} + \begin{bmatrix}
        -D_{s}^{2} \\ R_s^2
    \end{bmatrix} , \forall j \in \mathcal{N}_{i},
\end{align}

We extended the field-of-view CBFs in~\cite{bertoncelli2024directed} to include $\beta_{H}\in [\pi,2\pi)$ and our CBFs are defined as follows:
\begin{align}
    &b_{\mathrm{fov}}({}^i\mathbf{r}_j) = \nonumber \\
    &\begin{cases}
    \begin{bmatrix}
        \tan(\beta_H/2) & 1 \\
        \tan(\beta_H/2) & -1
    \end{bmatrix} \begin{bmatrix}
        {}^ix_j \\ {}^iy_j
    \end{bmatrix} 
    ,  \!&\text{if } \beta_{H}\in[0, \pi)\\
    \begin{bmatrix}
        1 & 0
    \end{bmatrix} \begin{bmatrix}
        {}^ix_j \\ {}^iy_j
    \end{bmatrix} 
        ,  &\text{if } \beta_{H} = \pi\\
    \begin{bmatrix}
        \tan(\pi - \frac{\beta_H}{2}) & \text{sign}({}^iy_j)
    \end{bmatrix} \begin{bmatrix}
        {}^ix_j \\ {}^iy_j
    \end{bmatrix} , &\text{if } \beta_{H}\in(\pi, 2\pi)
    \end{cases}
\end{align}
\newtext{The CBFs $b_{\mathrm{sr}}({}^i\mathbf{r}_j)$ force robot $i$ to maintain a minimum safety distance $D_s$ and a maximum distance equal to the sensing range $R_s$ from robot $j$, $\forall j \in \mathcal{N}_i$. Meanwhile, $b_{\mathrm{fov}}({}^i\mathbf{r}_j)$ force the robot to keep the neighbor robot $j$ inside the 2D visual cone with amplitude $\beta_H$. Combining them, we obtain: 
}
\begin{align}\label{eq:hocbf_formulation}
    &b({}^i\mathbf{r}_j) = \left[b_{\mathrm{sr}}({}^i\mathbf{r}_j); b_{\mathrm{fov}}({}^i\mathbf{r}_j)\right] \geq 0, \forall j \in \mathcal{N}_{i},
\end{align}
\newtext{constraining robots to keep all their neighbors inside the field of view while moving.}
Note that $b$ has a relative degree $q=2$ with respect to system dynamics~\eqref{eq:dynamics}. Therefore, we use HOCBFs to guarantee constraint satisfaction. Choosing $\alpha_1(b({}^i\mathbf{r}_j)) = \gamma_1 b^{(2\mu+1)}({}^i\mathbf{r}_j) $ and $\alpha_2(\psi_1({}^i\mathbf{r}_j)) = \gamma_2\psi_1^{(2\mu+1)}({}^i\mathbf{r}_j)$, for $\mu \in \mathbb{N}$, we can rewrite~\eqref{eq:hocbf_def} as:
\begin{multline}\label{eq:hocbf_input}
    L_f^2b(\cdot) + L_gL_fb(\cdot)\mathbf{u} + (2\mu+1)\gamma_1b^{2\mu}(\cdot)L_fb(\cdot) \\+ \gamma_2(L_fb(\cdot) + \gamma_1b^{(2\mu+1)}(\cdot))^{(2\mu+1)} \geq 0.
\end{multline}


We choose $\alpha_1(\cdot)$ and $\alpha_2(\cdot)$ as odd power functions, which belong to extended class $\mathcal{K}$ functions. 
Therefore, the designed HOCBF~\eqref{eq:hocbf_formulation} is also a HOCLBF\newtext{~\cite{xiao2021hoclbf}} and brings the system back into the safe set $\mathcal{C}$ if not already within (see~\cite[Theorem 2]{xiao2021hoclbf}).
We use this property in our controller to tolerate constraint violations and enable the robot to regain visual contact with its neighbors after temporary tracking loss.

\section{Trajectory and Control Generation with Safety Certification}
\label{sec:mpc_cbf}
Our algorithm generates the continuous-time trajectory and control certified by control barrier functions, utilizing piecewise splines. 
Our optimization problem solves for the piecewise $h$-th order Bézier curves.  The trajectory is defined as the piecewise Bézier curves and their first-order derivatives. The control inputs $\mathbf{u}(t)$ are defined as their second derivatives.
We choose a sufficiently large $h$ to generate a smooth control $\mathbf{u}(t)$. 
To satisfy the visual contact requirement, we impose the HOCBF constraints in~\eqref{eq:hocbf_input}, for any given $t$ in the horizon. 
The general form of our problem is formulated as follows:
\begin{subequations}\label{eq:12gen}
\begin{IEEEeqnarray}{rCl'rCl}
\argmin_{\boldsymbol{\mathcal{U}}} 
&~& \mathcal{J}_{\mathrm{cost}} \label{eq:general_cost}\\
\text{s.t.} 
&~& \dot{\mathbf{x}}(t) = A\mathbf{x}(t) + B{\mathbf{u}(t)} \label{eq:model_constraint}\\
&~& \frac{d^j f(0)}{dt^j} = \frac{d^j \mathbf{r}(t_{0})}{dt^j} , ~\forall j \in\{0, \ldots, C\} 
\label{eq:general_initial_state_constraint}\\
&~& f \text{ continuous up to derivative } C \label{eq:general_continuity_constraint}\\
&~& \small{\begin{aligned} A^{\mathrm{cbf}}\mathbf{u}(t) + \boldsymbol{b}^{\mathrm{cbf}}({}^{i}\hat{\mathbf{r}}_{j}(t|t_{0})) \geq 0, ~&\forall t\!\in\! [t_{0}, t_{0}\!+\!\tau] \\ & \forall j \!\in\! \mathcal{N}_{i} \end{aligned}} \label{eq:general_qp_cbf}\\
&~& \mathbf{a}_{\mathrm{min}} \preceq \mathbf{u}(t) \preceq \mathbf{a}_{\mathrm{max}}, ~\forall t\!\in\! [t_{0}, t_{0}\!+\!\tau]\\
&~& \mathbf{v}_{\mathrm{min}} \preceq \mathbf{v}(t) \preceq \mathbf{v}_{\mathrm{max}}, ~\forall t\!\in\! [t_{0}, t_{0}\!+\!\tau],
\end{IEEEeqnarray}
\end{subequations}
where $\preceq$ stands for element-wise less than or equal to, 
$t_{0}$ is the current time stamp, and $C$ is the highest order of derivatives required for continuity. 
The constraint~\eqref{eq:general_qp_cbf} is equivalent to~\eqref{eq:hocbf_input}, where $A^{\mathrm{cbf}} = L_gL_fb$, $\boldsymbol{b}^{\mathrm{cbf}} = L_f^2b + (2\mu+1)\gamma_1b^{2\mu}L_fb + \gamma_2(L_fb + \gamma_1b^{(2\mu+1)})^{(2\mu+1)}$. 
Note, ${}^{i}\hat{\mathbf{r}}_{j}(t|t_{0}) = \hat{\mathbf{r}}_{j}(t_{0}) - \mathbf{r}_{i}(t)$, as we can only obtain the current estimation of the neighbor $\hat{\mathbf{r}}_{j}(t_{0})$ in a communication-denied setting.

\begin{theorem}
Consider the HOCBF in~\eqref{eq:hocbf_formulation} and the set $\mathcal{C}:=\mathcal{C}_1 \cap \mathcal{C}_2$. 
Let $\alpha_{1}$, $\alpha_{2}$ be differentiable extended class $\mathcal{K}$ functions. If $\mathbf{x}(t_0)\in \mathcal{C}$, the controller $\mathbf{u}(t)$ from~\eqref{eq:12gen}, $\forall t\in [t_{0},t_{0}+\tau]$ renders $\mathcal{C}$ forward invariant. Otherwise, $\mathbf{u}(t)$ from~\eqref{eq:12gen}, $\forall t\in [t_{0},t_{0}+\tau]$ stabilizes system~\eqref{eq:dynamics} towards the set $\mathcal{C}$. 
\end{theorem}
\begin{proof}
$\mathbf{u}(t)$ is Lipschitz continuous as it is defined as the second-order derivative of the Bézier curve with sufficient continuity. 
The constraint~\eqref{eq:model_constraint} requires the state transition to obey the system model in~\eqref{eq:dynamics}. The constraints~\eqref{eq:general_qp_cbf} render the system safe in the horizon or stabilize the system to a safe set $\mathcal{C}$ following directly from Theorem 1.  
\end{proof}

Solving the above optimization, however, is intractable, as~\eqref{eq:general_qp_cbf} imposes constraints for all $t$ in continuous time, yielding an infinite number of constraints. Instead, we propose a discrete optimization scheme to approximate the solution, depicted in Fig.~\ref{fig:mpc-cbf}. 
Our approach imposes the HOCBF constraints only at time stamps sampled at fixed intervals over the horizon. The trajectory \newtext{is replanned in real-time and the most recent optimized trajectory} is executed only \newtext{up to} the first sampled time step. Since our optimization scheme acts similarly to an MPC with continuous-time control inputs, we name our algorithm model predictive control with control barrier functions, or MPC-CBF for short. 
\begin{figure}[t]
    \centering
    \includegraphics[width=0.5\linewidth]{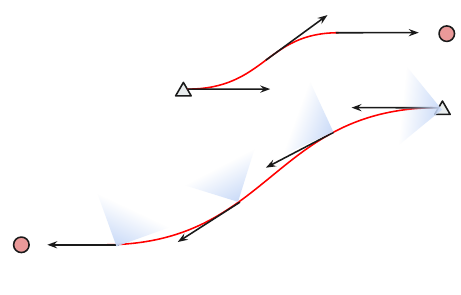}
    \caption{Robot navigates to its goal (red dot) with predicted field of views (blue triangles). MPC-CBF imposes constraints at sampled steps.} 
    \label{fig:mpc-cbf}
    \vspace{-2em}
\end{figure}
\subsection{Trajectory and Control Prediction Model}
We introduce the notation $\hat{(\cdot)}(k|t_{0})$, which represents the prediction of $(\cdot)(k|t_{0})$, given information at time $t_{0}$ and horizon $k \in \{0, \cdots, K-1\}$, where $(K - 1)\delta = \tau$. Here $\delta$ is the duration of each discrete time step. 
The prediction of system output $\hat{\mathbf{y}}(k|t_{0})$ is the optimized piecewise Bézier curve resulting from~\eqref{eq:12gen}; its first and second-order derivatives $\hat{\mathbf{v}}(k|t_{0})$ and $\hat{\mathbf{u}}(k|t_{0})$ can be computed in closed form.
By definition, the predicted trajectory $\hat{\mathbf{x}}(k|t_{0}) = [\hat{\mathbf{y}}(k|t_{0}); \hat{\mathbf{v}}(k|t_{0})]$ is the solution of~\eqref{eq:dynamics} given control inputs $\hat{\mathbf{u}}(k|t_{0})$.

\subsection{HOCBF Constraints and Relaxation}
To satisfy the safety requirements, we approximate HOCBF constraints in~\eqref{eq:general_qp_cbf} using a sampling-based approach. 
We constrain $\hat{\mathbf{u}}(k|t_{0})$ at discrete time steps over horizon by
\begin{align}
    A^{\mathrm{cbf}}\hat{\mathbf{u}}(k|t_{0}) + \boldsymbol{b}^{\mathrm{cbf}}({}^{i}\hat{\mathbf{r}}_{j}(k|t_{0})) &\geq 0, ~\forall j \in \mathcal{N}_{i} \nonumber\\ &\forall k\in \{0, \ldots, K-1\}. 
    \label{eq:hocbf}
\end{align}
As samples increase, constraints~\eqref{eq:hocbf} approach HOCBF constraints in~\eqref{eq:general_qp_cbf}. 
Note that the prediction $\hat{\mathbf{y}}(k|t_{0})$ depends on the decision variables, hence ${}^{i}\hat{\mathbf{r}}_{j}(k|t_{0})$ does too. The constraint~\eqref{eq:hocbf} is nonlinear, due to the nonlinearity of $A^{\mathrm{cbf}}$ and $\boldsymbol{b}^{\mathrm{cbf}}$. 
In Sec.\ref{sec:SQP}, we propose a linear approximation of the constraint and solve the problem with SQP. 

More neighbors increase the number of constraints, which leads to infeasibility. We relax~\eqref{eq:hocbf} with slack variables for distant robots, which do not pose a danger of collisions. We define slack variables $\epsilon_j\geq 0$, for $j\in \mathcal{N}_{i}$ (see Sec.~\ref{sec:cost_functions}).
\subsection{Collision Avoidance Constraints}
The safety distance HOCBF cannot guarantee collision avoidance in the horizon, as the robot can only estimate the current relative position of its neighbors $^{i}\hat{\mathbf{r}}_{j}(t_{0})$ without knowing their plans. 
We use a separating hyperplane approach, similar to~\cite{pan2025hierarchical}, to guarantee collision avoidance in belief space. 
A function $L(\mathcal{A}, \mathcal{B})$ computes a separating half-space $\hat{\mathcal{H}}_{r}\!\! :=\!\! \left\{\mathbf{r}\in \mathcal{W}\mid\boldsymbol{w}^{\top}_{r}\mathbf{r} + b_{r} \leq 0\right\}$, where $\boldsymbol{w}_{r}$ and $b_{r}$ are the weights and bias of the half-space, respectively, and $\mathcal{A}$ and $\mathcal{B}$ are the convex hulls representing the robots. We compute Voronoi-cell separation between $\mathbf{r}_{i}$ and $\mathbf{r}_{j}$ as $\hat{\mathcal{H}}_{r}$. By buffering the half-space by an offset $b_{r}^{'} = b_{r} + \mathrm{max}_{\boldsymbol{y}\in \mathcal{R}(\mathbf{0})} \boldsymbol{w}_{r}^{\top} \boldsymbol{y}$, we obtain that the safety corridor consists of $\mathcal{H}_{r}$ for robot $i$. The Bézier curve $f_{i}$ generated at the negative side of $\mathcal{H}_{r}$ guarantees collision avoidance with its neighbor's belief. 
We can write this constraint in the form
\begin{align}
    A^{\mathrm{col}}_{i}\boldsymbol{u}_{i,j} + \boldsymbol{b}^{\mathrm{col}}_{i} \leq 0, ~&\forall i \in \left\{ 0, \ldots, P-1 \right\} \nonumber \\
    &\forall j \in \left\{ 0, \ldots, h \right\}.
\end{align}
\subsection{Output and Derivatives Continuity}
To guarantee continuity of the system output and its derivatives, we need to impose continuity between the splines, thus adding the following constraints,
\begin{align}
    \frac{d^j f_{i}\left(\tau_{i}\right)}{d t^j}=\frac{d^j f_{i+1}(0)}{d t^j}, ~& \forall i \in\{0, \ldots, P-2\} \nonumber \\ 
    &\forall j \in\{0, \ldots, C\}.
\end{align}

\subsection{Physical Limits}
We require limits on the derivatives due to physical constraints.  
The derivatives of Bézier curves are confined within the convex hull of the derivative's control points. 
This approach, however, is overly conservative~\cite{mercy2017spline}. In~\cite{honig2018trajectory}, the duration of the Bézier curve is iteratively rescaled until the physical constraints are satisfied. 
Inspired by~\cite{luis2020online}, we propose an approach that leverages our discrete optimization scheme. We bound the values of 
$\hat{\mathbf{v}}(k|t_{0})$, $\hat{\mathbf{u}}(k|t_{0})$ in the horizon, 
\begin{align}
    \mathbf{v}_{\mathrm{min}} &\preceq \hat{\mathbf{v}}(k|t) \preceq \mathbf{v}_{\mathrm{max}}, ~\forall k\in \{0, \ldots, K-1\},\\
    \mathbf{a}_{\mathrm{min}} &\preceq \hat{\mathbf{u}}(k|t) \preceq \mathbf{a}_{\mathrm{max}}, ~\forall k\in \{0, \ldots, K-1\}.
\end{align}

\subsection{Cost Functions}
\label{sec:cost_functions}
We optimize the predicted trajectory and control inputs considering different objectives according to the task. 
\subsubsection{Goal Cost}
The trajectory should navigate the robot towards its goal $\mathbf{y}_{d}\in \mathbb{R}^{4}$. We penalize the squared distance between the last $\kappa$ sampled predictions $\hat{\mathbf{y}}(k|t_{0})$ and the goal,
\begin{align}
    \mathcal{J}_{\mathrm{goal}} = \sum_{k = K-\kappa}^{K-1} \omega_{k}\left\Vert \hat{\mathbf{y}}(k|t_{0}) - \mathbf{y}_{d}\right\Vert_{2}^{2},
\end{align}
where $\omega_{k}$ is the weight for $k$-th sample. 

\subsubsection{Control Effort Cost} We minimize the weighted sum of the integral of the square of the norm of derivatives, 
\begin{align}
    \mathcal{J}_{\mathrm{effort}} = \sum_{j=1}^{C} \theta_{j} \int_{t_{0}}^{t_{0}+\tau} \left\Vert\frac{d^{j}}{dt^{j}} f(t)\right\Vert_{2}^{2} \, dt,
\end{align}
where $\theta_{j}$ is the weight of the order of derivatives. 

\subsubsection{Priority Cost} 
We prioritize tightening the HOCBF constraints for the nearest neighbors to maintain visual contact and prevent impending collisions. 
We derive a confidence ellipsoid ${\mathcal{R}}^{95}_{j}$ containing the real position $\mathbf{r}_j$ with $95\%$ probability from the particle filter. We find the distance $d_{ij}$ between robot $i$ and ${\mathcal{R}}^{95}_{j}$ following the solution in~\cite{catellani2023distributed}.
Sorting  neighbors based on the distance $d_{ij}$ (from the closest to the farthest one), we obtain an ordered set $\overline{\mathcal{N}}_i$, and prioritize the satisfaction of the HOCBF constraints on robots that are believed to be closer.

Priority assignment is achieved by adding slack variables $\epsilon_j$ with exponentially decaying weights $\xi_j = \Omega \cdot \gamma_s^{j}$ as a cost function, 
where $\Omega \in \mathbb{R}_{>0}$ is the cost factor and $\gamma_s \in (0,1)$ is the decay factor. Therefore, the cost can be defined as
\begin{equation}
    \mathcal{J}_{\mathrm{prior}} = \sum_{j\in\overline{\mathcal{N}}_{i}} \xi_j\epsilon_j.
\end{equation}
Slack variables let robot $i$ temporarily lose visual contact with a distant neighbor $j$, but reestablish it 
when the uncertainty grows and the ellipsoid ${\mathcal{R}}^{95}_{j}$ is close. 

\section{Solving MPC-CBF Optimization}
\label{sec:SQP}
As mentioned in Sec.~\ref{sec:mpc_cbf}, the HOCBF constraints in~\eqref{eq:hocbf} are nonlinear. 
We propose a linear surrogate of HOCBF constraints and use the SQP to solve the proposed problem with a QP solver. 
The SQP iteratively solves MPC-CBF. In each iteration, we use the prediction from the previous QP as the surrogate of the states to compute $A^{\mathrm{cbf}}$ and $\boldsymbol{b}^{\mathrm{cbf}}$. This decouples the dependency between the prediction and decision variables, so that $A^{\mathrm{cbf}}$ and $\boldsymbol{b}^{\mathrm{cbf}}$ can be treated as constants. 
The initial QP is solved with HOCBF constraint only on the observable state at time $t_{0}$ and predicts $\hat{\mathbf{x}}_{0}(k|t_{0})$ and $\hat{\mathbf{u}}_{0}(k|t_{0})$. Here, the subscription indicates the QP iteration index. However, the predicted trajectory and control inputs do not necessarily satisfy HOCBF constraints in the horizon. 
For the $\nu$-th QP iteration, we substitute ${}^{i}\mathbf{r}_{j}(k|t_{0})$ with the predicted ${}^{i}\hat{\mathbf{r}}_{j,\nu-1}(k|t_{0})$ from the previous QP, for $\nu = 1, \ldots, \mathcal{V}-1$. The $\nu$-th QP is formulated as follows:
\begin{subequations}
\begin{IEEEeqnarray}{rCl'rCl}
\argmin_{\boldsymbol{\mathcal{U}}} 
&~& \mathcal{J}_{\mathrm{effort}} + \mathcal{J}_{\mathrm{goal}} + \mathcal{J}_{\mathrm{prior}} \label{QPcost}\\
\text{s.t.} 
&~& \frac{d^j f(0)}{dt^j} = \frac{d^j \mathbf{r}}{dt^j} , ~\forall j \in\{0, \ldots, C\} \label{QPconst:init_state}\\
&~& \begin{aligned} \frac{d^j f_{i}\left(T_{i}\right)}{d t^j}=\frac{d^j f_{i+1}(0)}{d t^j}, ~& \forall i \in\{0, \ldots, P\!-\!2\} \\ & \forall j \in\{0, \ldots, C\}\end{aligned} \label{QPconst:continuity}\\
&~& \begin{aligned} A^{\mathrm{col}}_{i}\boldsymbol{u}_{i,j} + \boldsymbol{b}^{\mathrm{col}}_{i} \leq 0, ~& \forall i \in\{0, \ldots, P-1\} \\ & \forall j \in\{0, \ldots, h\}\end{aligned} \label{QPconst:safety_corridor}\\
&~& \begin{aligned} A^{\mathrm{cbf}}\hat{\mathbf{u}}(k|t_{0}) &+ \boldsymbol{b}^{\mathrm{cbf}}(^{i}\hat{\mathbf{r}}_{j,\nu-1}(k|t_{0})) + \epsilon_j \geq 0, \\ &\forall j\in \mathcal{N}_{i},~ \forall k\in \{0, \ldots, K-1\} 
\end{aligned} \label{QPconst:cbf}\\
&~&\mathbf{v}_{\mathrm{min}} \preceq \hat{\mathbf{v}}(k|t_{0}) \preceq \mathbf{v}_{\mathrm{max}}, \forall k\!\!\in\!\! \{0, \ldots, K\!-\!1\}\\
&~&\mathbf{a}_{\mathrm{min}} \preceq \hat{\mathbf{u}}(k|t_{0}) \preceq \mathbf{a}_{\mathrm{max}}, \forall k\!\!\in\!\! \{0, \ldots, K\!-\!1\}\\
&~& \epsilon_{j} \geq 0, ~\forall j\in \mathcal{N}_{i}.
\end{IEEEeqnarray}
\end{subequations}
\normalsize
Note that the robot estimates neighbors' positions only at $t_{0}$. Adding visual constraints for the entire horizon based on this estimate leads to an overly conservative plan. Instead, we satisfy HOCBF constraints only up to $K_{r}$ steps in the horizon.

\section{Simulation Results}
We define two sets of instances. In ``Circle" instances, robots are initialized uniformly on a circle with antipodal goals; their start and goal headings face the circle's center. In ``Formation" instances, robots are initialized in grids and demanded to move forward;  start and goal headings are set to $0$ yaw. 

To reflect the uncertainty in the system dynamics, Gaussian noise is added to the system output and velocity, i.e., $\mathbf{y}(k|t_{0}) \sim \mathcal{N}\left(\hat{\mathbf{y}}(k|t_{0}), \sigma_{\mathbf{y}}^{2}\mathbf{I}\right)$, $\mathbf{v}(k|t_{0})\sim \mathcal{N}\left(\hat{\mathbf{v}}(k|t_{0}), \sigma_{\mathbf{v}}^{2}\mathbf{I}\right)$, where $\mathcal{N}(\boldsymbol{\mu},\sigma\mathbf{I})$ denotes a multivariate Gaussian distribution with mean $\boldsymbol{\mu}$ and a diagonal covariance matrix $\sigma\mathbf{I}$. We set $\sigma_{\mathbf{y}}\!=\!0.001$, and $\sigma_{\mathbf{v}}\!=\!0.01$. In this work, we fix the height of the robots. We set different $\beta_{H}$ to demonstrate the property of our algorithm. We limit the acceleration in range $[-10, 10]\unit{m/s^{2}}$ in the x-y plane, and velocity in range $[-3, 3]\unit{m/s}$ for ``circle" instances and $[-0.5, 0.5]\unit{m/s}$ for ``formation" instances. We set the yaw acceleration and yaw rate limits as $[-\pi, \pi]\unit{rad/s^{2}}$ and $[-\frac{5}{6}\pi, \frac{5}{6}\pi]\unit{rad/s}$ respectively. To expedite the computation and respect the visual contact constraints, we set $K_{r}=2$, $\mathcal{V} = 2$ in the SQP solver. We set the number of pieces $P=3$ for the piecewise spline, the degree of Bézier curves $h=3$ with duration $\tau_{i}=0.5\unit{s}$, for $i=1,2,3$, and require the highest order of continuity $C=3$. In the MPC-CBF algorithm, we set the discrete sample interval $\delta=0.1\unit{s}$. \newtext{In simulation, the replanning happens every $0.1\unit{s}$.} For the particle filtering, we set the number of particles to $N_{p}=100$ (initialized uniformly randomly in the workspace), the process covariance to $0.25\mathbf{I}$, the measurement covariance $R_{\mathrm{m}}$ to  $0.05\mathbf{I}$, and the penalty factor for particles inside the field of view to $\varepsilon = 0.1$. The cost factor of slack variables is $\Omega = 1000$. The collision shape of the robot is defined as an axis-aligned bounding box in range $[-0.2, 0.2]\unit{m}$ for both x-y dimensions. 
\newtext{For the baseline, we implemented the controller from~\cite{catellani2023distributed}, extending it to an HOCBF with a double integrator under the same velocity and acceleration limits; a PD controller provides the desired input, and the control loop runs every $0.1\unit{s}$ in simulation.
}
\subsection{Simulation in Circle Instances}
We use the following criteria for evaluation:\newline
    \textbf{Success Rate}: success is defined as all robots reach their goal areas and stay within them without collisions.\newline
    \textbf{Makespan}: time at which the last robot reaches its goal area.\newline
    \textbf{Percentage of Neighbors in FoV}: average percentage of neighbors the robot keeps in visual contact over the makespan.\newline 
Goals may not satisfy the visual-contact requirement. Our method compensates for visual contact and thus could lead to deviations from the goal. Thus, we consider the robot to complete its task when it reaches a goal area and stays within. 

We set cost coefficients $\omega_{k}=10$ for $k = K-\kappa,\ldots,K-1$, $\theta_{j} = 1$ for $j=1\ldots C$, and $\kappa=3$. The snapshots in Fig.~\ref{fig:sim_circle} are typical routing of our control strategy in the ``circle" instance with $5$ robots and a $\beta_{H}=\frac{2}{3}\pi$ field of view. 
\newtext{As a demonstration of controller robustness,} note that the robot colored with green trajectory, when it loses visual contact with neighbors in Fig.~\ref{fig:sim_circle}(a), changes its position and heading in Fig.~\ref{fig:sim_circle}(b) to regain detection. 
The sensitivity of heading adjustments is controlled by the slack variable decay factor $\gamma_{s}$. A small $\gamma_{s}$ prioritizes tracking the closest neighbor in the belief space, resulting in more aggressive heading responses. A large $\gamma_{s}$ tends to track more neighbors, leading to less aggressive heading responses. An aggressive heading response can lead to an inefficient strategy due to frequent heading changes. An insensitive heading tends to overlook impending collisions, leading to actual collisions.
\begin{figure}[tb]
    \centering
    \subfloat[Time = 4.4\unit{s}]{\includegraphics[width=0.11\textwidth]{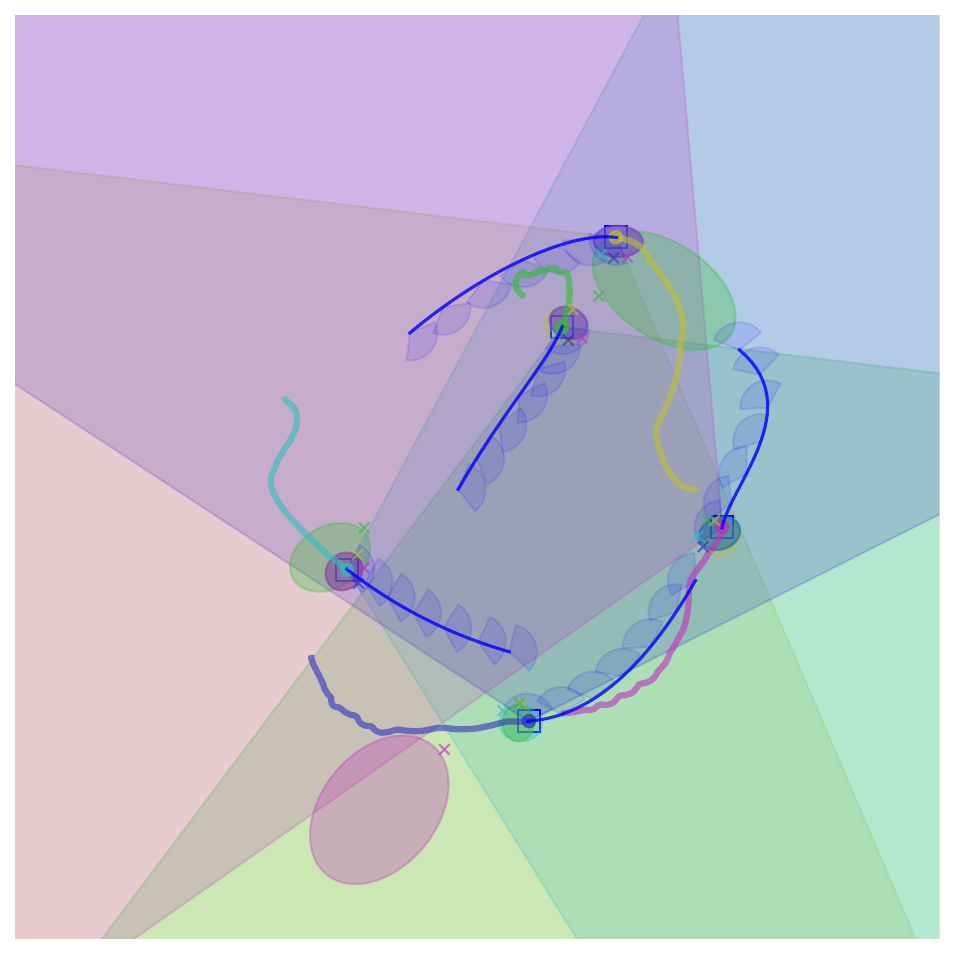}}
    \subfloat[Time = 6.8\unit{s}]{\includegraphics[width=0.11\textwidth]{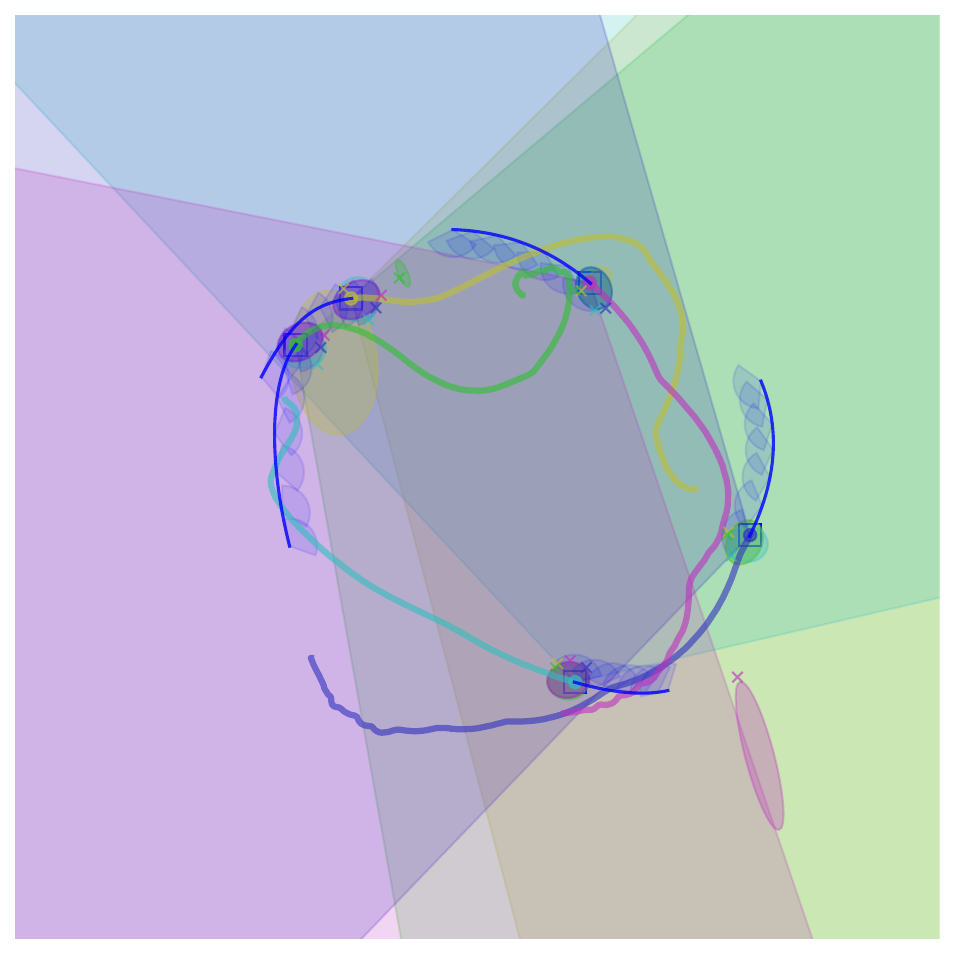}}
    \subfloat[Time = 11.0\unit{s}]{\includegraphics[width=0.11\textwidth]{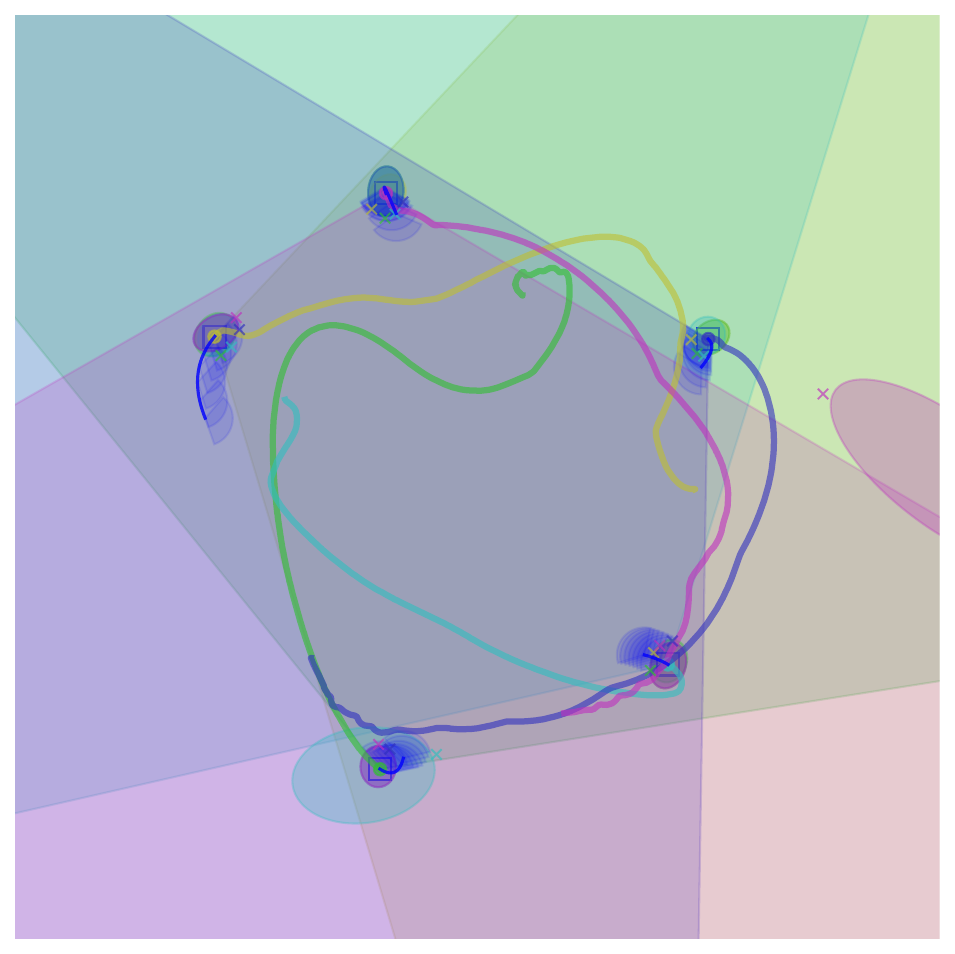}}
    \subfloat[Time = 15.0\unit{s}]{\includegraphics[width=0.11\textwidth]{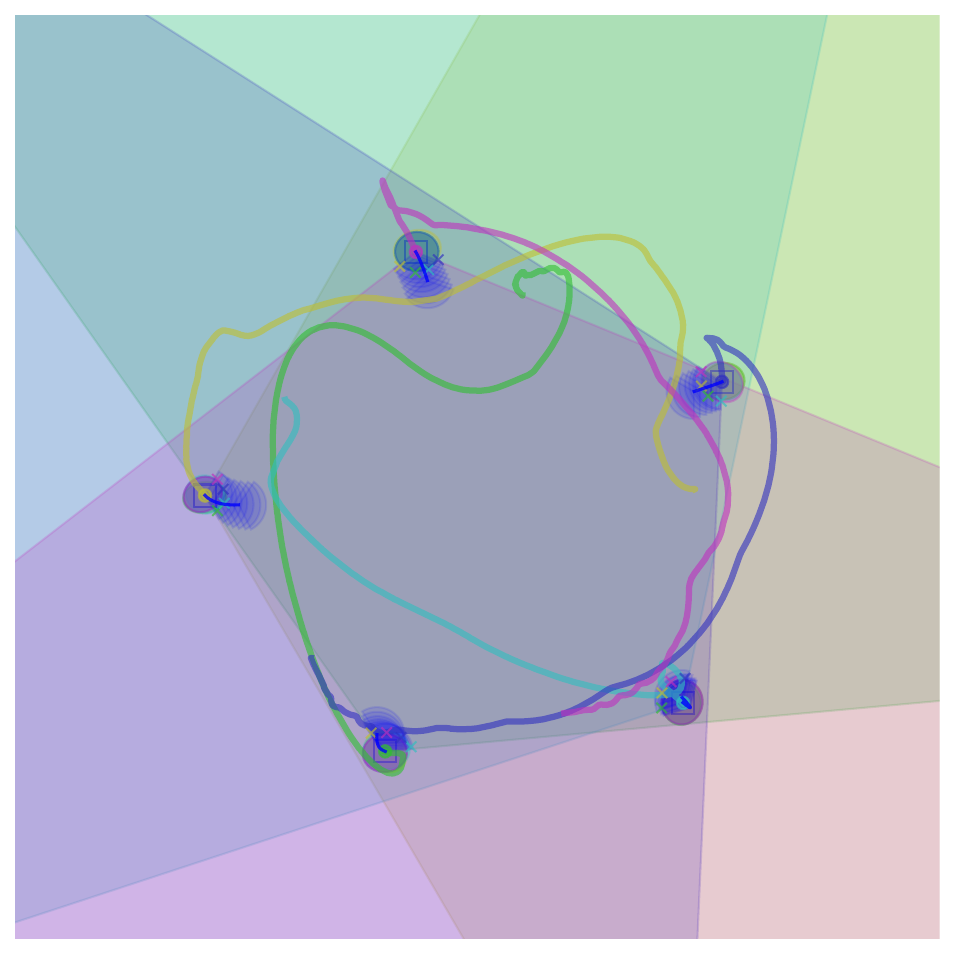}}
    \caption{Snapshots for 5 robots in the circle instance. The ovals are 95\% confidence ellipsoids of estimation (the source of estimations is indicated by colors). The predicted output is depicted as blue curves and purple field of views. The path is shown as a solid line.}
    \label{fig:sim_circle}
    \vspace{-1em}
\end{figure}
\begin{figure}[tb]
    \centering
    {\includegraphics[width=0.44\textwidth]{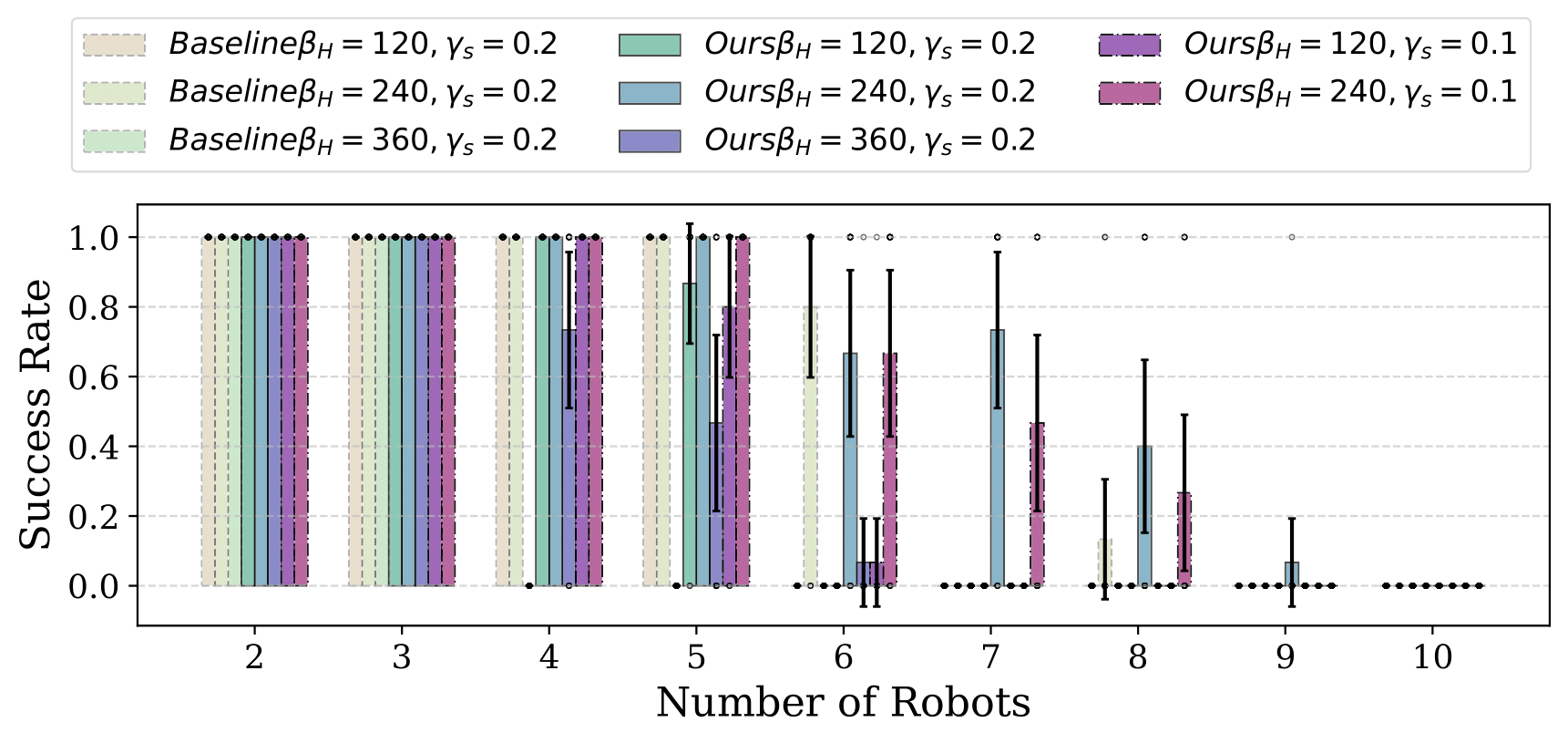}}\\
    {\includegraphics[width=0.44\textwidth]{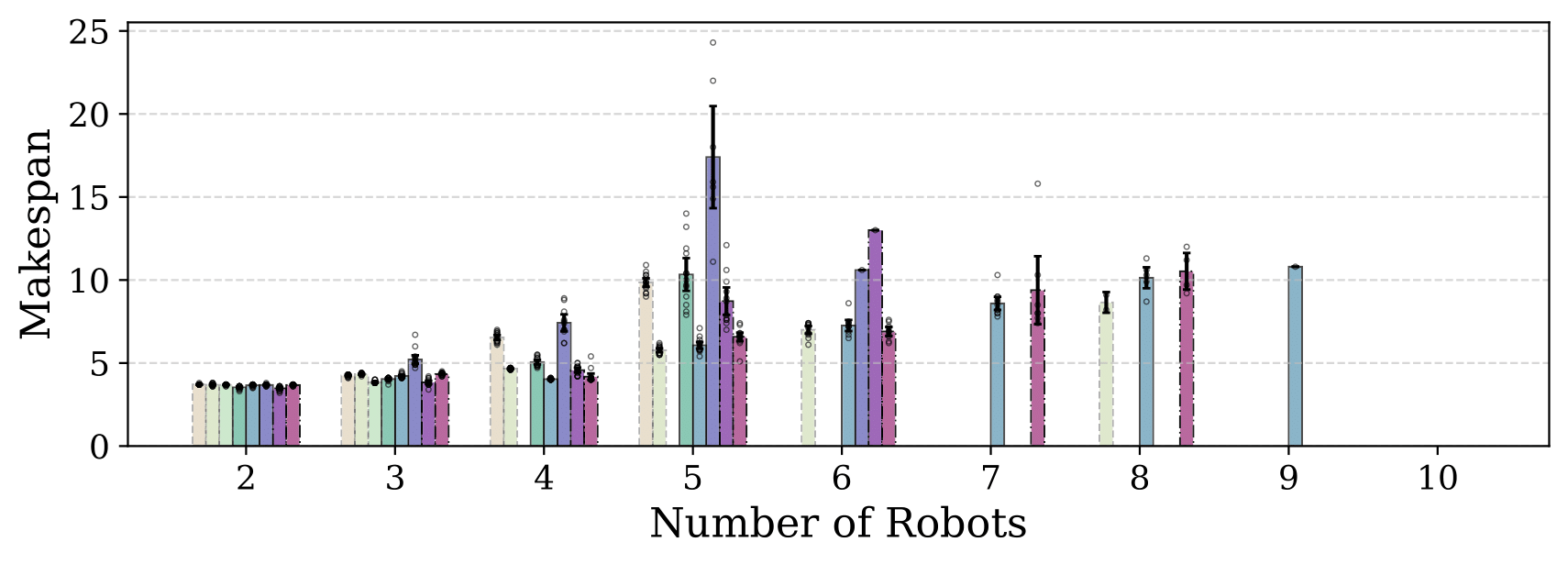}}\\
    {\includegraphics[width=0.44\textwidth]{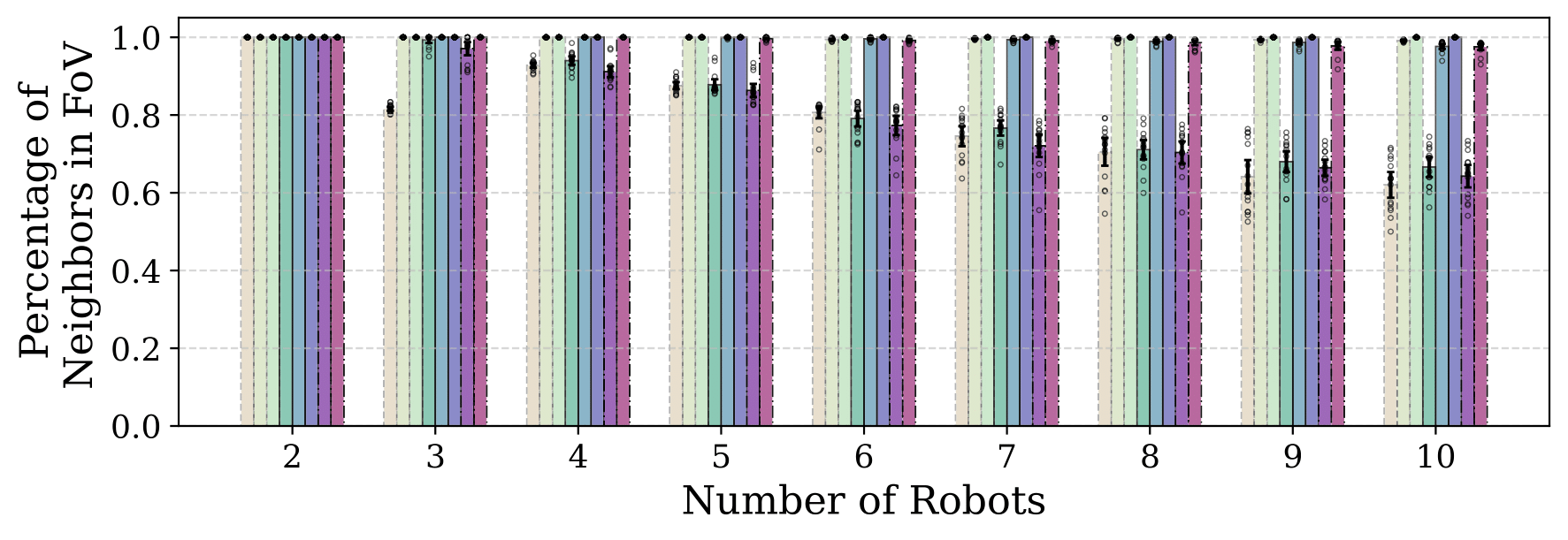}}
    \caption{Performance of our algorithm across $\beta_H$, $\gamma_s$, and different robot counts in ``circle'' instances. Bars show means, error bars show 95\% confidence intervals over 15 trials.}
    \label{fig:circle_quantity}
    \vspace{-2em}
\end{figure}

Figure~\ref{fig:circle_quantity} reports  quantitative performance of our control strategy with different $\beta_{H}$ in $[\frac{2}{3}\pi, \frac{4}{3}\pi, 2\pi]$ and slack variable decay factors $\gamma_{s}$ in $[0.1, 0.2]$. Note that, $\beta_{H}=2\pi$ always satisfies the field-of-view constraints. We show \textit{task success rate}, \textit{makespan} (excluding failure instances), and \textit{percentage of neighbors in FoV}. 
We notice that for a small number of robots (fewer or equal to $6$ robots in the ``Circle" instance), our controller has a similar success rate compared to the baseline. As we increase the number of robots, our algorithm outperforms the baseline with different $\beta_{H}$ and achieves successful executions when the baseline failed to complete the task at all. 
The reason is that the baseline, as a reactive controller, responds more aggressively to the imminent collisions and does not maintain a large safety distance. As a result, it is more sensitive to estimate uncertainty and exposes robots to a higher risk of collision once the neighbor detections are lost. For a small $\beta_{H}$, the control challenge is to maintain neighbors in the field of view, thus providing the latest estimation to avoid collision, while reaching the desired goal. As the field of view increases, maintaining visual contact with neighbors becomes easier, and robots tend to take the shortest path, leading to a decrease in the \textit{makespan}. However, it leads to 
deadlocks and potential collisions due to uncertainty in the estimation. 
The challenge of deadlocks in multi-robot planning falls outside the current scope. Modern MAPF-based path/trajectory planning addresses deadlock problems even in large-scale operations~\cite{pan2024hierarchical, pan2025hierarchical} \newtext{and limited communication scenarios~\cite{maoudj2024improved}}. 
Despite the drop in \textit{task success rate}, we note the \textit{percentage of neighbors in FoV} remains above $60\%$ as we scale up the number of robots even with $\beta_{H}=\frac{2}{3}\pi$. Our control strategy maintains the same level of visual contact with neighbors compared to the baseline controller while improving the success rate. \newtext{Additionally, we report that less aggressive heading adjustments, i.e., a larger $\gamma_s$, prevent robots from repetitively searching for neighbors, resulting in a higher success rate in ``Circle" instances. 
}
\subsection{Simulation in Formation Instances}
In ``formation" instances, we initialize all robots in grids $1\unit{m}$ apart in the x-y direction. All robots are initialized with $0$ yaw. Goals are $12\unit{m}$ to the right of the starts, each with a yaw of $0$. We set the cost coefficients $\omega_{k}=300$ for $k = K-\kappa,\ldots,K-1$, $\theta_{j} = 1$ for $j=1\ldots C$, and $\kappa=3$. In Fig.~\ref{fig:sim_formation}, we show a typical result of MPC-CBF on $4$ robots. 
Robots form and maintain visual contact with neighbors at all times, demonstrating the controller robustness. 
We summarize the quantitative results in Fig.~\ref{fig:formation_quantity}. Since robots move in the same direction, unlike ``Circle" instances, tasks result in fewer collisions and deadlocks. We notice an improved scalability and success rate compared to ``Circle" instances. Overall, our controller outperforms the baseline across different numbers of robots. With a small field of view, the baseline satisfies the field-of-view constraints while compromising the goal-reaching capability. In contrast, our control strategy retains goal-reaching capability while maintaining visual contact with neighbors. With an omnidirectional field of view, the baseline suffers from collision due to estimation uncertainty as the number of robots increases. Increasing the field of view improves the success rate for both the baseline and our approach. In addition, a larger $\gamma_{s}$ works better for ``formation" instances. The \textit{makespan} does not show significant changes with different $\beta_{H}$ or $\gamma_{s}$, mainly because the task is less challenging regarding navigation (the \textit{makespan} metric is omitted due to the space limit). The \textit{percentage of neighbors in FoV} remains above $60\%$ even with $\beta_{H}=\frac{2}{3}\pi$. Our controller maintains equivalent visual contact quality without compromising the goal-reaching capability.
\newtext{We report the runtime of the proposed perception control stack with different numbers of robots in Table~\ref{table:runtime}.
\begin{figure}[tb]
    \centering
    \subfloat[Time = 0.0\unit{s}]{\includegraphics[width=0.14\textwidth]{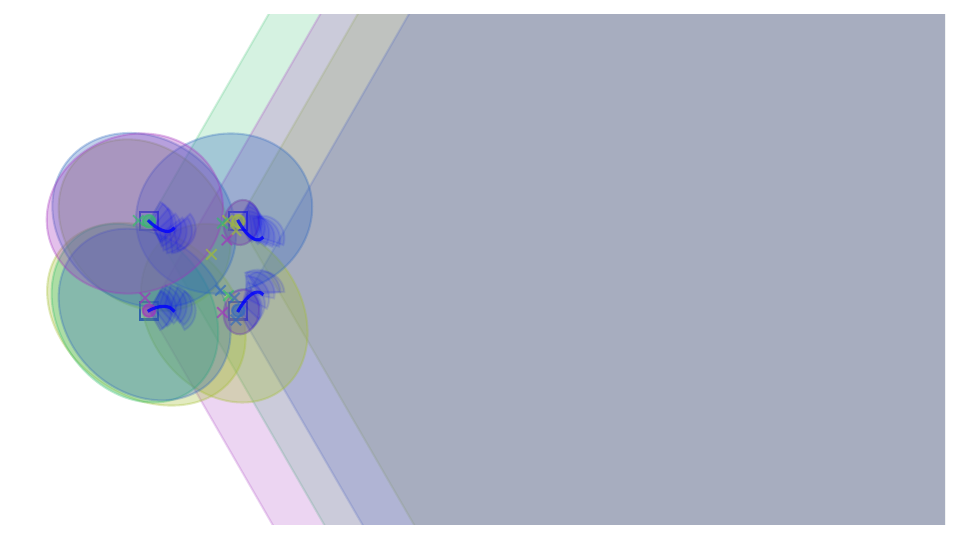}}
    \subfloat[Time = 10.0\unit{s}]{\includegraphics[width=0.14\textwidth]{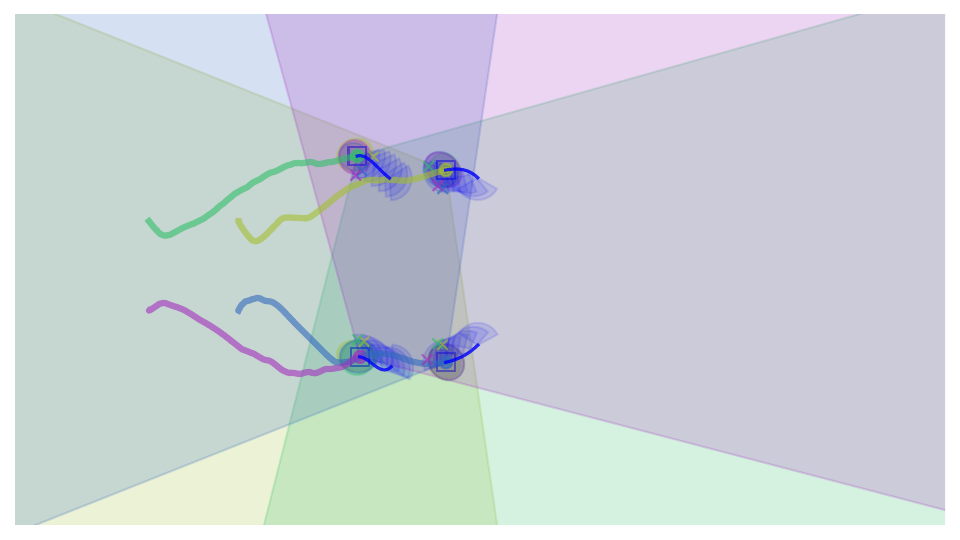}}
    \subfloat[Time = 27.0\unit{s}]{\includegraphics[width=0.14\textwidth]{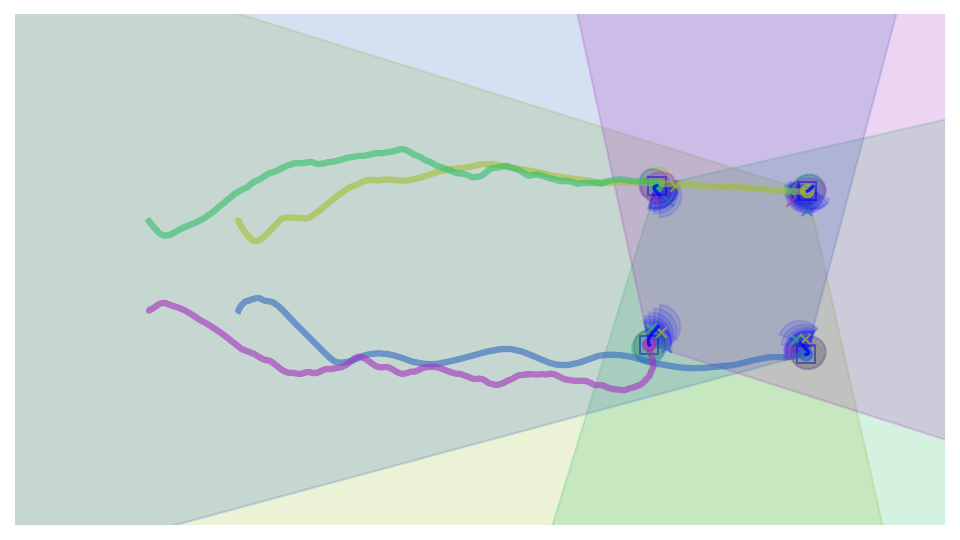}}
    \caption{4 robots navigating in a formation instance. Their start and goal yaws are set as 0. The ovals are 95\% confidence ellipsoids of estimation. The robot forms and maintains visual contact with its neighbors.}
    \label{fig:sim_formation}
    \vspace{-1em}
\end{figure}
\begin{figure}[tb]
    \centering
    {\includegraphics[width=0.44\textwidth]{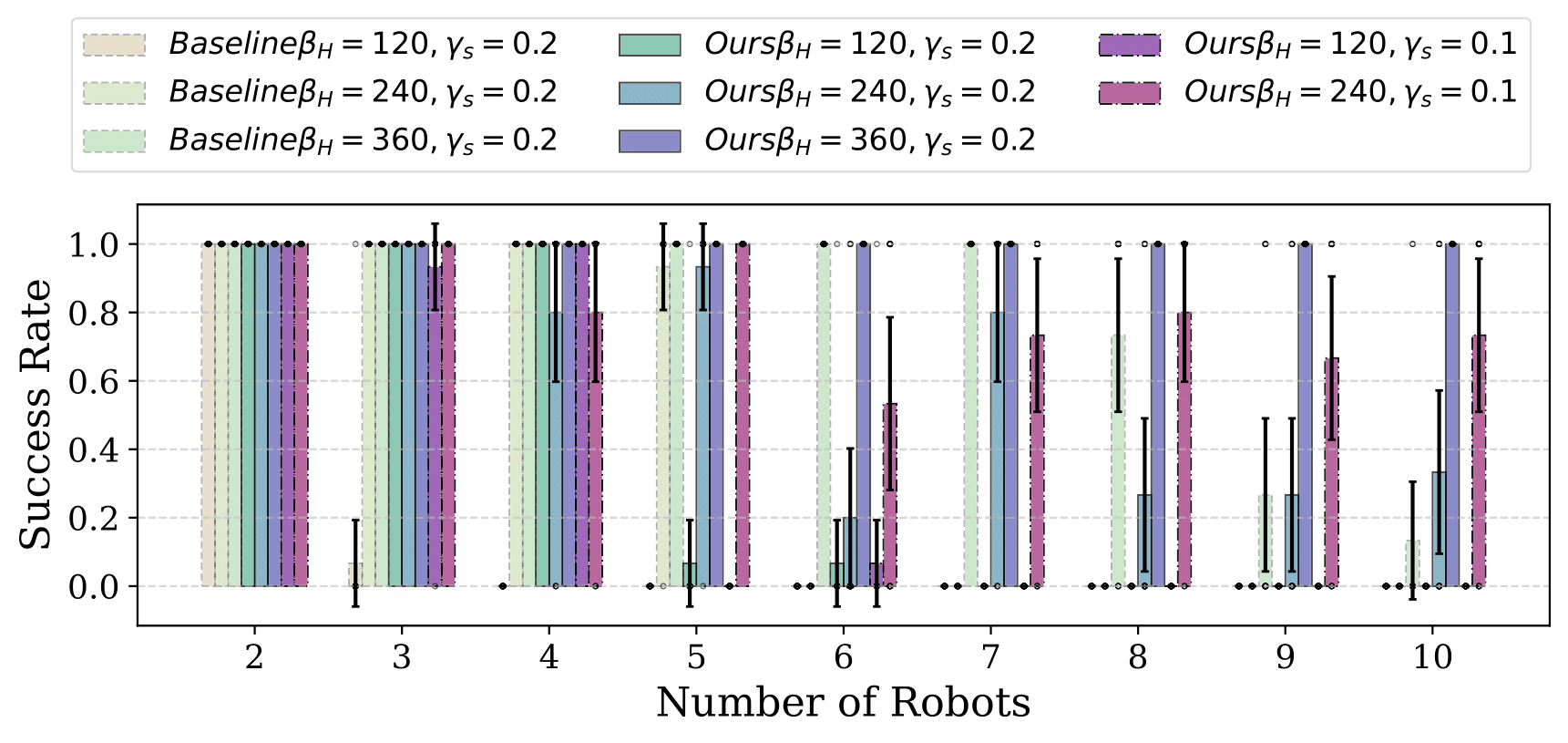}}\\
    {\includegraphics[width=0.44\textwidth]{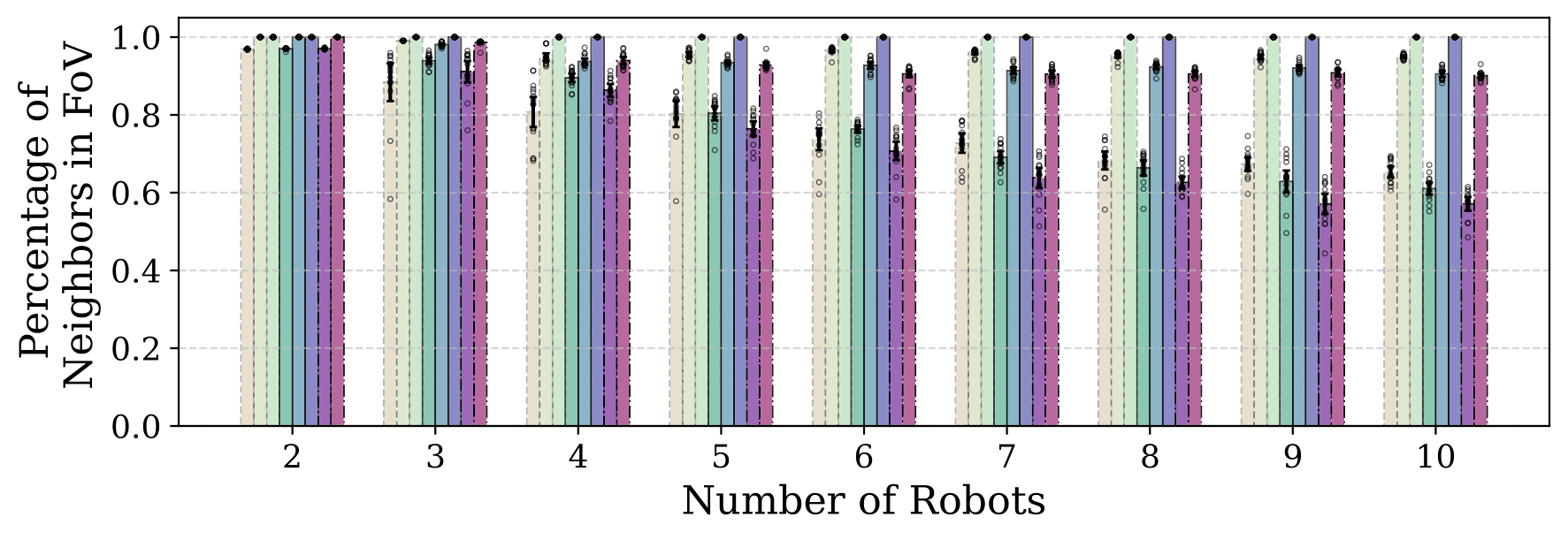}}
    \caption{Performance on different numbers of robots in ``formation" instances. The statistics are obtained in the same way in the ``circle" instances.}
    \label{fig:formation_quantity}
    \vspace{-1em}
\end{figure}
\begin{table}[h]
\centering
\scalebox{0.64}{
\begin{tabular}{|c|c|c|c|c|c|c|c|c|c|}
\hline
\multicolumn{1}{|c|}{Instance/number of Robot} & \multicolumn{1}{c|}{$2[\unit{ms}]$} & \multicolumn{1}{c|}{$3[\unit{ms}]$} & \multicolumn{1}{c|}{$4[\unit{ms}]$} & \multicolumn{1}{c|}{$5[\unit{ms}]$} & \multicolumn{1}{c|}{$6[\unit{ms}]$} & \multicolumn{1}{c|}{$7[\unit{ms}]$} & \multicolumn{1}{c|}{$8[\unit{ms}]$} & \multicolumn{1}{c|}{$9[\unit{ms}]$} & \multicolumn{1}{c|}{$10[\unit{ms}]$} \\
\hline
Baseline ``Circle" & 1.06 & 2.09 & 3.46 & 4.81 & 6.12  & 7.26 & 8.91 & 9.84 & 11.16 \\
\hline
Baseline ``Formation" & 0.76 & 2.04 & 3.43 & 4.87 & 5.59 & 6.94 & 8.44 & 9.79 & 10.56 \\
\hline
MPC-CBF ``Circle" & 5.87 & 10.04 & 13.05 & 18.55 & 21.95 & 23.64 & 28.30 & 31.15 & 35.98 \\
\hline
MPC-CBF ``Formation" & 5.21 & 7.84 & 10.69 & 14.32 & 17.74 & 19.39 & 22.78 & 25.46 & 28.65 \\
\hline
\end{tabular}
}
\caption{The runtime of baseline and MPC-CBF in ``Circle" and ``Formation" instances. The statistics are averaged over $200$ iterations.}
\label{table:runtime}
\vspace{-2em}
\end{table}
}
\subsection{Sensitivity Analysis} \label{sec:sensitivity_analysis}
\newtext{\textbf{Sample interval $\delta$}. We approximate the HOCBF constraints in the horizon using a sampling-based approach. The forward invariance property, however, can be violated due to large sample intervals. We conduct experiments with different sample intervals $\delta$ over a $0.2\unit{s}$ horizon on the ``Circle" instance with $4$ robots and compared with baseline and MPC-CBF without HOCBF constraints in the horizon, i.e., $K_{r}=1$. All methods use slack variables to avoid optimization failure, but with a large slack cost $\Omega \!=\! 1e\!+\!20$.} 
\begin{table}[h]
\centering
\scalebox{0.8}{
\begin{tabular}{|c|c|c|}
\hline
\multicolumn{1}{|c|}{Methods} & \multicolumn{1}{c|}{Percentage of Neighbors in FoV} & \multicolumn{1}{c|}{Runtime} \\
\hline
Baseline & 92.98\% & 12.87$\unit{ms}$ \\
\hline
MPC-CBF, $K_{r}=1$ & 90.87\% & 46.98$\unit{ms}$ \\
\hline
MPC-CBF, $\delta=0.05\unit{s}$ & 97.86\% & 118.17$\unit{ms}$ \\
\hline
MPC-CBF, $\delta=0.1\unit{s}$ & 98.01\% & 68.98$\unit{ms}$ \\
\hline
MPC-CBF, $\delta=0.2\unit{s}$ & 97.98\% & 57.57$\unit{ms}$ \\
\hline
\end{tabular}
}
\caption{The percentage of neighbors in FoV and the runtime of baseline and MPC-CBF as the HOCBF constraint sample intervals change. The statistics are averaged over $15$ trials.}
\label{table:sample_interval}
\vspace{-1em}
\end{table}
\newtext{We use the \textit{Percentage of Neighbors in FoV} to evaluate the efficacy of the set forward invariance. From Table~\ref{table:sample_interval}, we note that all algorithms cannot maintain visual contact at all times. Reactive controllers, such as baseline and MPC-CBF, $K_{r}=1$, perform significantly worse compared to MPC-CBF with horizon HOCBF constraints. The reactive controllers have more aggressive path and heading adjustments, leading to more failures in visual maintenance. Decreasing sample intervals $\delta$ has an insignificant influence on performance, but increases the runtime. We choose $\delta=0.1\unit{s}$, as it imposes denser CBF constraints without a significant runtime penalty. 

\textbf{Detection Noises and Delays}. We evaluate the efficacy of our algorithm under noisy and delayed detections. We use the same experiment setup as the above $4$ robots experiment. A Gaussian noise $\mathcal{N}(\mathbf{0}, \sigma_{\mathrm{noise}}^{2}\mathbf{I})$ is added to the detected neighbor state. From Table~\ref{table:detection_noises_and_delays}, we notice that the \textit{Percentage of Neighbors in FoV} decreases as delay and noise increase. Despite performance degradation, our algorithm maintains $100\%$ success rate on the $4$ robots experiment.} 
\begin{table}[h]
\centering
\scalebox{0.8}{
\begin{tabular}{|c|c|c|}
\hline
\multicolumn{1}{|c|}{Setup} & \multicolumn{1}{c|}{Percentage of Neighbors in FoV} & \multicolumn{1}{c|}{Success Rate} \\
\hline
Delay = 0.05$\unit{s}$ & 91.52\% & 100.00\% \\
\hline
Delay = 0.1$\unit{s}$ & 84.85\% & 100.00\% \\
\hline
Delay = 0.2$\unit{s}$ & 74.23\% & 100.00\% \\
\hline
$\sigma_{\mathrm{noise}}$ = 0.1 & 98.17\% & 100.00\% \\
\hline
$\sigma_{\mathrm{noise}}$ = 0.2 & 97.29\% & 100.00\% \\
\hline
\end{tabular}
}
\caption{The percentage of neighbors in FoV and the success rate as the detection delay and noise change. The statistics are averaged over $15$ trials.}
\label{table:detection_noises_and_delays}
\vspace{-2em}
\end{table}
\section{Physical Experiment}
We build PX4-based UAVs and track their position and orientation with Vicon during the experiment. The motion capture can be replaced by onboard state estimation, such as GPS and a magnetometer during outdoor flights or VIO/LIO during indoor flights, providing a communication-free setup for localization. 
The UAV is equipped with a Jetson Xavier and a RealSense D435 camera, where $\beta_{H} \approx \frac{1}{4}\pi$. 
\newtext{We model the quadrotor as a double integrator and send position, velocity, and acceleration commands to its onboard flight controller. This model mismatch may violate forward invariance in physical experiments. The quadrotor model requires a third-order CLF-CBF~\cite{xiao2023safe}, which we leave for future work. We enforce field-of-view constraints in the x-y plane, noting they can be extended to 3D by adding constraints in the x-z plane. We use onboard Apriltag detection in a single-robot experiment, shown in the accompanying video.} 
\newtext{In the multi-robot experiment, due to unstable Apriltag detection, we simulate detection capabilities using Vicon. In other words,} the robot receives the relative location (detection) of its neighbor from Vicon when that neighbor is in its sensing region. \newtext{We note that relying on Vicon deviates from the decentralized assumption.} 
\newtext{Onboard detection (e.g.,~\cite{ge2022vision}) would introduce noise, delay and missed detections. We discussed the influence of noise and delay in Sec.~\ref{sec:sensitivity_analysis}. Our algorithm handles the missed detections as discussed in Sec.~\ref{sec:filter}.} We demonstrate visual contact during navigation with a two-robot experiment: two UAVs maintain visual contact with each other at all times, as shown in Fig.~\ref{fig:multi_robot_experiments}. 
\begin{figure}[t]
    \centering
    \subfloat
    {\includegraphics[width=0.14\textwidth]{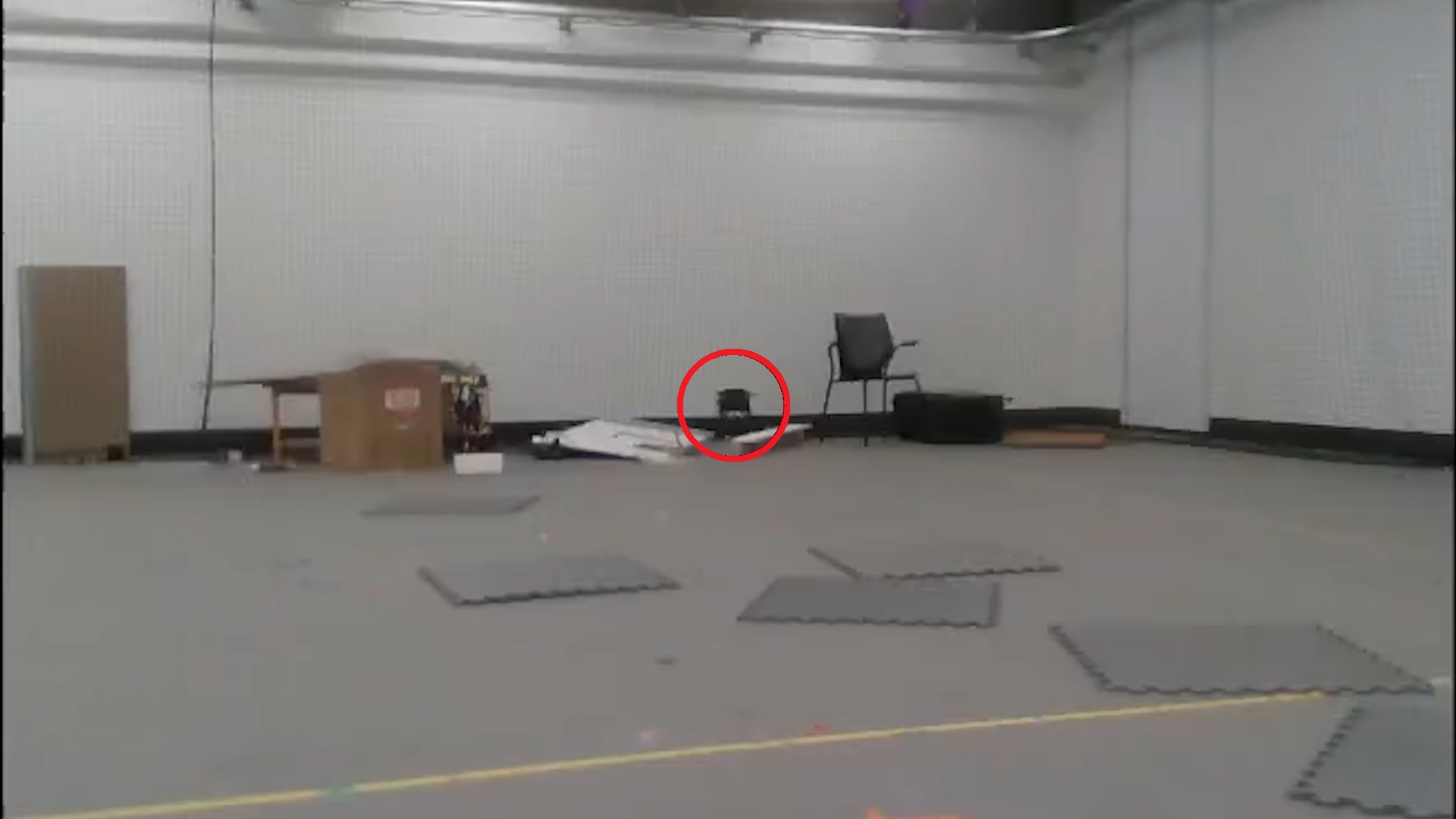}}
    \subfloat
    {\includegraphics[width=0.14\textwidth]{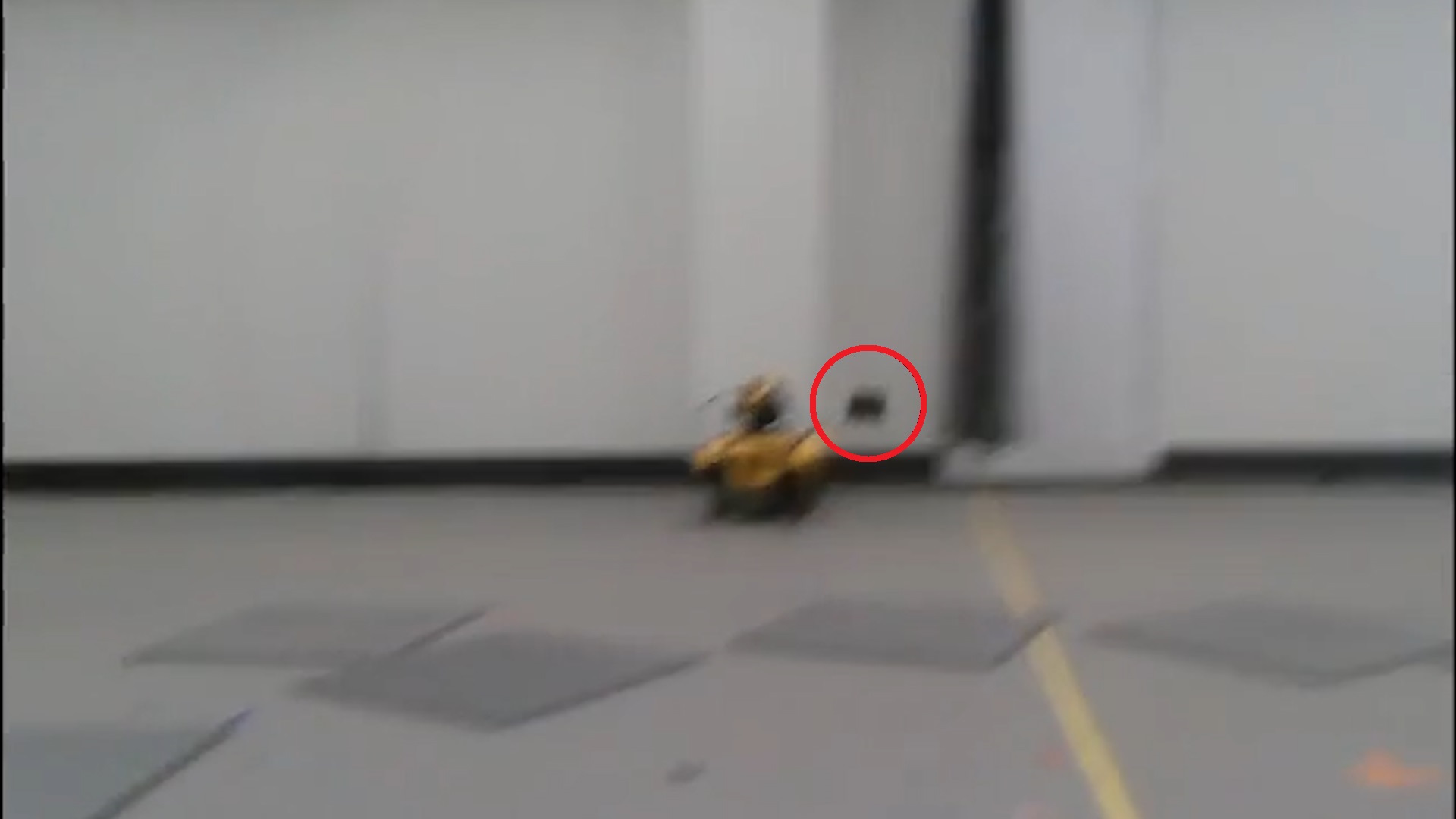}}
    \subfloat
    {\includegraphics[width=0.14\textwidth]{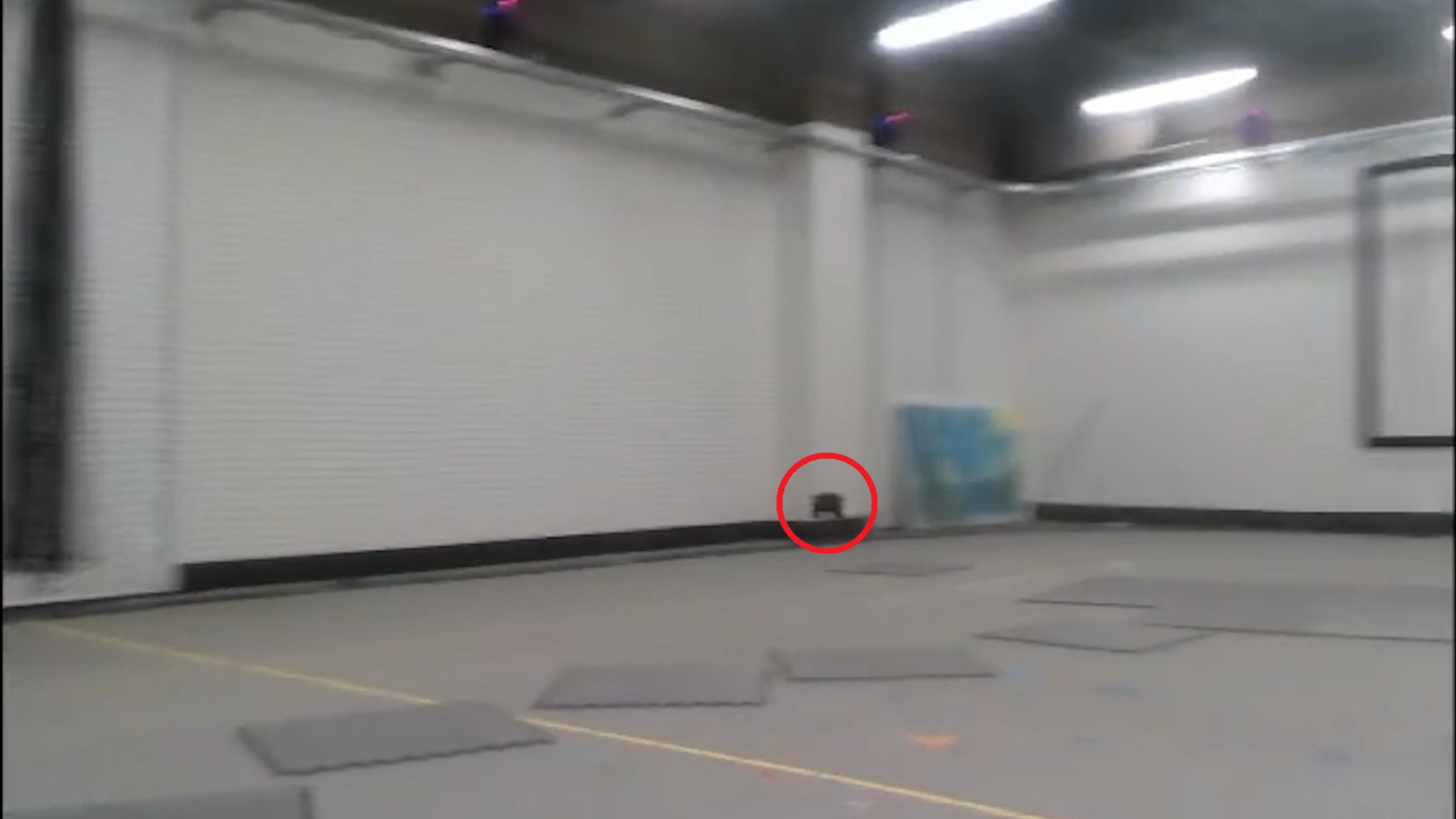}}\\
    \vspace{0.1em}
    \setcounter{subfigure}{0}
    \subfloat[Time = 4.2\unit{s}]{\includegraphics[width=0.14\textwidth]{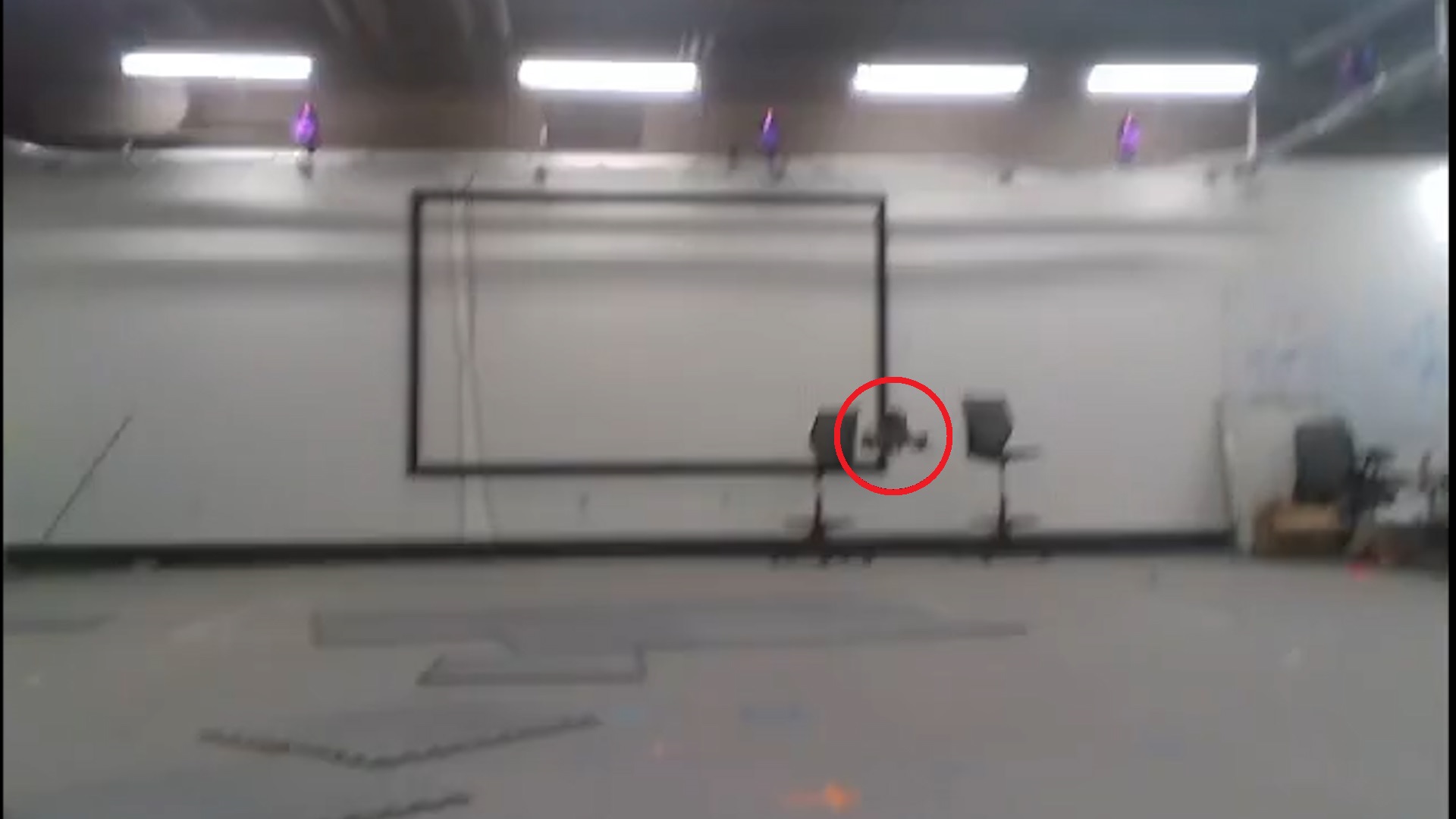}}
    \subfloat[Time = 10.8\unit{s}]{\includegraphics[width=0.14\textwidth]{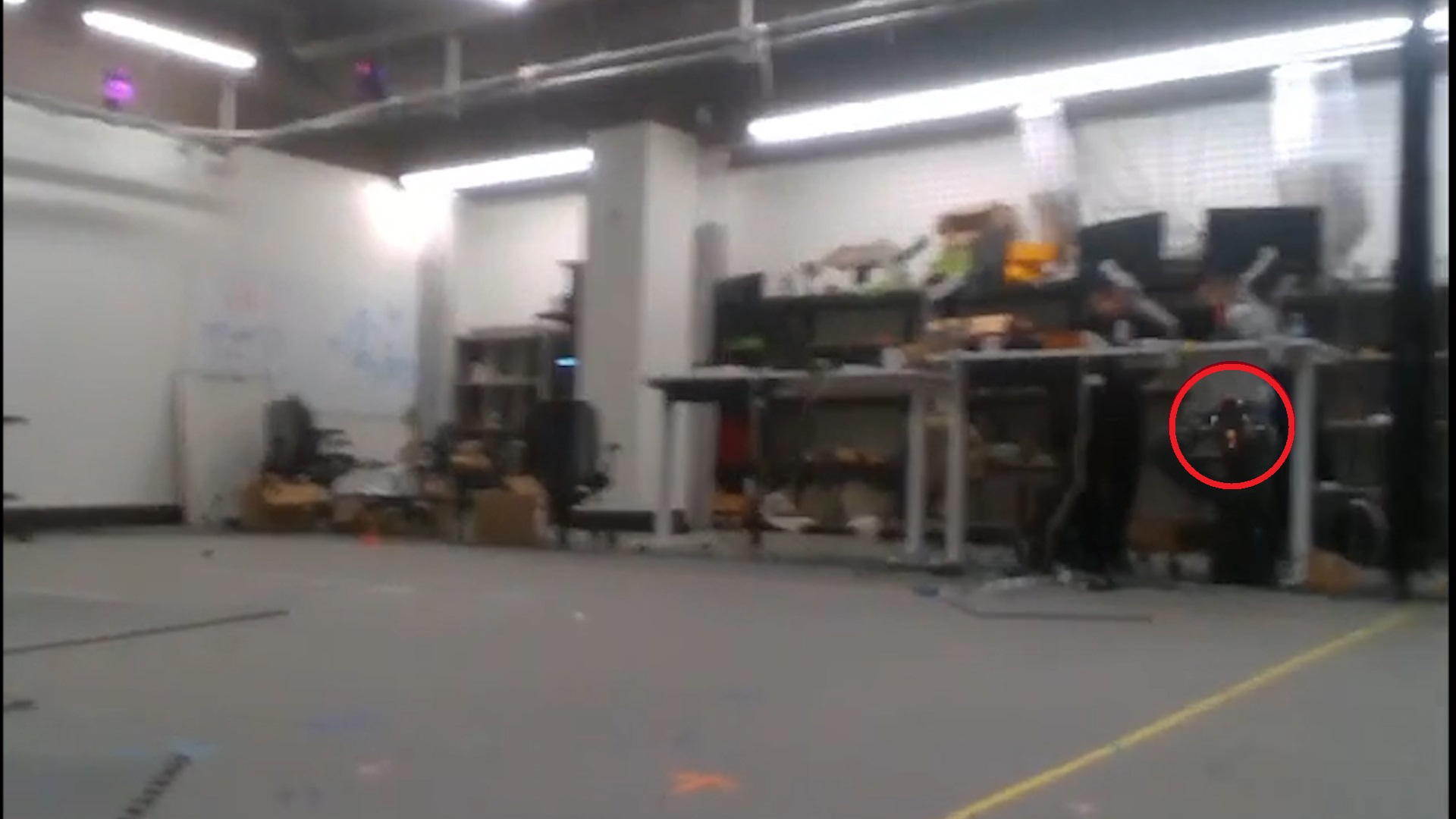}}
    \subfloat[Time = 14.3\unit{s}]{\includegraphics[width=0.14\textwidth]{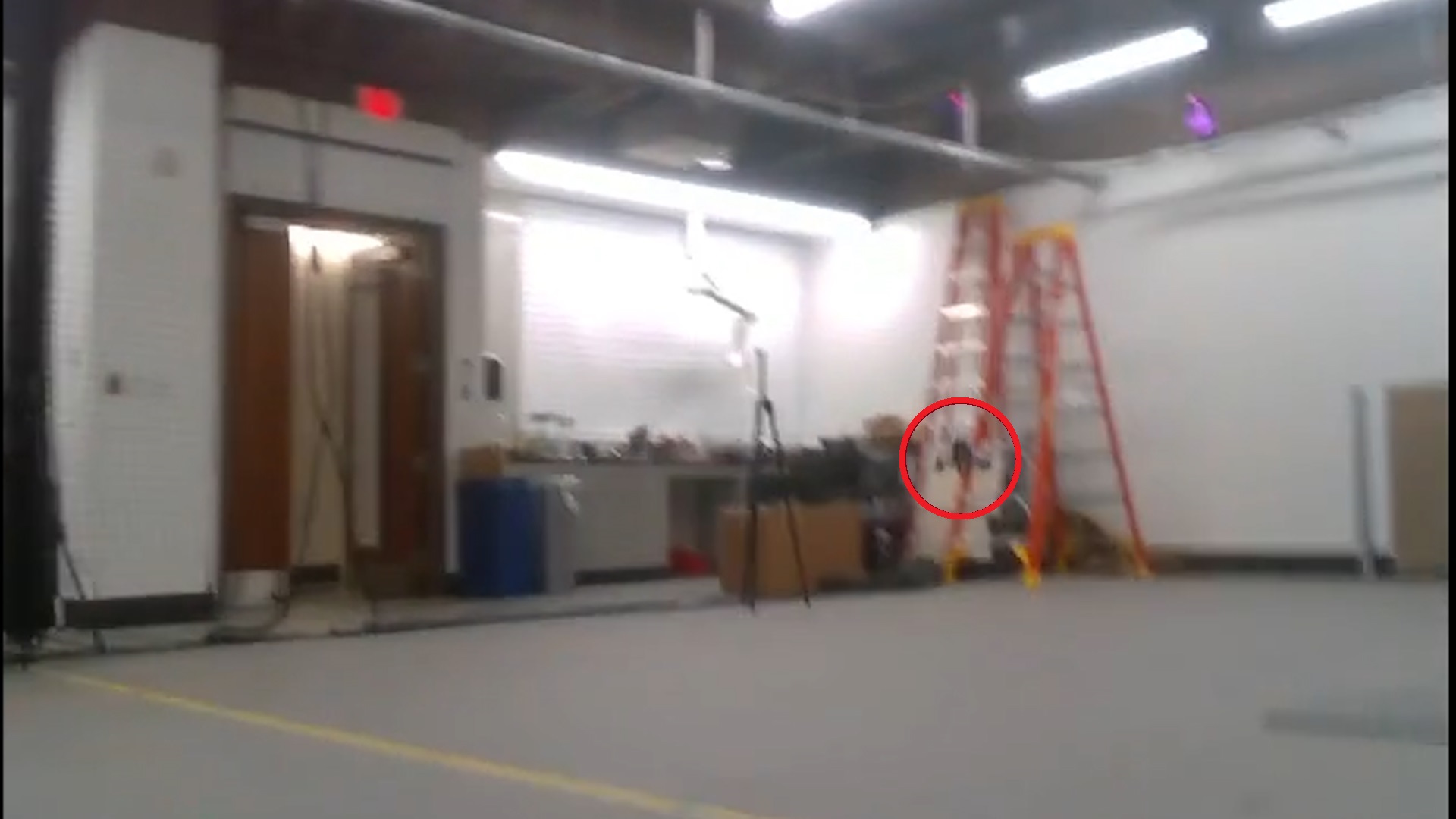}}
    \caption{Snapshots of first-person-view camera recording in the two-robot experiment. 
    The visual contact constraints are satisfied throughout the flight. The other robot is circled in red.}
    \label{fig:multi_robot_experiments}
    \vspace{-1em}
\end{figure}
\vspace{-1em}
\section{Conclusion}
We address online distributed coordination without communication. Robots navigate to their goals while estimating the locations of neighbors by maintaining visual contact. The proposed strategy is robust to temporary tracking loss and able to regain tracking. 
We propose MPC-CBF, a discrete optimization framework to approximate the certified solution. In addition, we propose an efficient SQP to solve MPC-CBF with a QP solver. We verify the efficacy and scalability of our algorithm with $10$ robots in simulation and physical experiments with $2$ custom-built UAVs with onboard cameras. In future work, we aim to extend to non-planar motion, and develop an adaptive controller that is resilient to model perturbation caused by external forces, such as downwash~\cite{kiran2024downwash}. 

\bibliographystyle{IEEEtran}
\bibliography{IEEEabrv, refs}

\end{document}